\documentclass[final,12pt]{clear2025} 

\usepackage{bm}
\usepackage{enumitem}
\usepackage{pgf,tikz} \usetikzlibrary{arrows, fit, shapes, arrows.meta, decorations, positioning, matrix}
\usepackage{tkz-graph}
\pgfdeclarelayer{bg}
\pgfsetlayers{bg,main}

\usepackage{multirow}
\usepackage{array}
\usepackage{mathrsfs}

\tikzset{dir/.style = { - >, thick}}
\tikzset{bidir/.style = {< - >, thick}}
\tikzset{nondir/.style = {Circle[open] - Circle[open], thick}}
\tikzset{anydir/.style = {Square - >, thick}}
\tikzset{anypart/.style = {Square - Circle[open], thick}}
\tikzset{partdir/.style = {Circle[open] - >, thick}}
\tikzset{any/.style = {Square - Square, thick}}

\newcommand{\G}{\mathcal{G}}

\newcommand{\p}{\mathcal{P}}
\newcommand{\M}{\mathcal{M}}
\newcommand{\D}{\mathcal{D}}

\newcommand{\I}{\mathcal{I}}
\newcommand{\K}{\mathcal{K}}
\newcommand{\X}{\mathbf{X}}
\newcommand{\V}{\mathbf{V}}
\newcommand{\E}{\mathbf{E}}
\newcommand{\R}{\mathbf{R}}
\newcommand{\F}{\mathbf{F}}

\newcommand{\adj}[2]{\mathrm{adj}_{{#1}}({#2})}

\newcommand{\pa}[2]{\mathrm{pa}_{{#1}}({#2})}

\newcommand{\an}[2]{\mathrm{an}_{{#1}}({#2})}
\newcommand{\de}[2]{\mathrm{de}_{{#1}}({#2})}
\newcommand{\possan}[2]{\mathrm{possan}_{{#1}}({#2})}
\newcommand{\possde}[2]{\mathrm{possde}_{{#1}}({#2})}
\newcommand{\dsep}[2]{\mathrm{dsep}_{{#1}}({#2})}
\newcommand{\possdsep}[2]{\mathrm{poss.dsep}_{{#1}}({#2})}
\newcommand{\possdsepp}[2]{\mathrm{pds}_{{#1}}({#2})}
\newcommand{\indep}{\perp\mkern-10mu\perp}
\newcommand{\any}{*\mkern-6mu-\mkern-7mu*}

\newcommand{\rightpart}{\circ\mkern-7mu\rightarrow}
\newcommand{\leftpartany}{*\mkern-8.5mu-\mkern-7mu\circ}
\newcommand{\rightpartany}{\circ\mkern-7mu-\mkern-8.5mu*}
\newcommand{\leftany}{\leftarrow\mkern-9.25mu *}
\newcommand{\rightany}{* \mkern-9.25mu\rightarrow}
\newcommand{\nondir}{\circ\mkern-6.5mu-\mkern-6.5mu\circ}
\DeclareMathAlphabet\mathbfcal{OMS}{cmsy}{b}{n}

\title[Causal discovery with tiered background knowledge and overlapping datasets]{
Constraint-based causal discovery with tiered background knowledge and latent variables in single or overlapping datasets}
\usepackage{times}

\clearauthor{
 \Name{Christine W Bang} \Email{bang@leibniz-bips.de}\\
 \AND
 \Name{Vanessa Didelez} \Email{didelez@leibniz-bips.de}\\
 \addr Leibniz Institute for Prevention Research and Epidemiology -- BIPS, Bremen, Germany\\ Faculty of Mathematics and Computer Science, University of Bremen, Bremen, Germany
}

\begin{document}

\maketitle

\begin{abstract}
  In this paper we consider the use of tiered background knowledge within constraint based causal discovery. Our focus is on settings relaxing causal sufficiency, i.e.\ allowing for latent variables which may arise because relevant information could not be measured at all, or not jointly, as in the case of multiple overlapping datasets. We first present novel insights into the properties of the `tiered FCI' (tFCI) algorithm. Building on this, we introduce a new extension of the IOD (integrating overlapping datasets) algorithm incorporating tiered background knowledge, the `tiered IOD' (tIOD) algorithm. 
  We show that under full usage of the tiered background knowledge tFCI and tIOD are sound, while simple versions of the tIOD and tFCI are sound and complete.
  We further show that the tIOD algorithm can often be expected to be considerably more efficient and informative than the IOD algorithm even beyond the obvious restriction of the Markov equivalence classes. We provide a formal result on the conditions for this gain in efficiency and informativeness.  Our results are accompanied by a series of examples illustrating the exact role and usefulness of tiered background knowledge.
\end{abstract}

\begin{keywords}
  Causal inference, graphical models, multi-cohort studies, temporal structure
\end{keywords}

\section{Introduction}
This work aims at exploiting tiered background knowledge for constraint-based causal discovery. The focus, here, is  on relaxing causal sufficiency, i.e. allowing for latent variables which may arise, for instance, because relevant information could not be measured jointly or not at all. The former occurs in cases where datasets stem from e.g. different studies that have some but not all measured variables in common. We then speak of multiple (partially) overlapping datasets.

In classical constraint-based causal discovery algorithms, such as the PC  \citep{spirtes2000causation} and FCI algorithms \citep{spirtes1999fci, zhang2008completeness}, the entire joint independence structure of the data is available, where the FCI algorithm allows for latent variables, i.e. unmeasured common causes of measured variables. The Integrating Overlapping Datasets (IOD) algorithm \citep{tillman2011learning} extends the FCI to multiple datasets, i.e. multiple independence structures that are each marginal over those variables in the other datasets that they do not contain. 
Allowing for latent variables clearly induces a lower degree of identifiability of causal relations. It is therefore important to efficiently use any available background knowledge  to improve identifiability. We focus on tiered background knowledge which arises in temporal data, in particular in cohort studies where multiple datasets occur for instance in multi-cohort designs \citep{o2022better}.

As the IOD algorithm builds on the FCI algorithm, we first consider incorporating tiered background knowledge into the FCI and build on this to propose an IOD algorithm exploiting tiered background knowledge; we refer to these as the `tiered' FCI/IOD, or `tFCI/tIOD' algorithms. 
While the tFCI algorithm has been implemented \citep{scheines1998, tfci, petersen2023} and applied \citep{lee2022causal} before, to our knowledge, only a few formal results have been proven. Therefore, we provide a formal description of the tFCI algorithm  and, as a step towards showing correctness of the tIOD, we show correctness of the tFCI.

While algorithms like the FCI represent lack of identification in form of non-directed edges, the output of the IOD algorithm is often more complex: Not only does it allow for different edges within an equivalence class (represented as PAG), but it can also comprise of different equivalence classes. Background knowledge is then useful in two regards: To restrict the equivalence classes and to reduce the number of possible equivalence classes. The proposed tIOD algorithm outputs a set of graphs that are all consistent with the information available in all datasets and with the known tiered ordering. Moreover, we show that with tiered background knowledge we can reduce the number of of possible equivalence classes even without orienting any additional edges. 

The motivation for our work is to be able to use data from multiple cohort studies (multi-cohorts) to learn causal structures across long time spans. Cohort studies are common in life sciences, such as life course epidemiology \citep{kuh2004life}, but the results here are valid for any type of data that has a tiered ordering. 
Tiered background knowledge is known to not only improve the informativeness of the estimated graphs \citep{bang2023we}, but also the accuracy of discovery algorithms in finite samples  \citep{bang2024improving}. 
While these previous results are limited to the case of causal sufficiency and a single dataset, the present work relaxes these assumptions.

\subsection{Related work}
A precursor of the IOD algorithm was the Integrating Overlapping Networks (ION) algorithm \citep{tillman2008integrating}. Building on the IOD, \citet{dhir2020integrating} propose orienting directed edges in the IOD algorithm using a bivariate causal discovery algorithm, but assume causal sufficiency and a particular data generating model. \citet{huang2020causal} introduce two approaches that learn the entire DAG over multiple overlapping datasets under assumptions of linearity and non-Gaussianity. 
 Other work on causal discovery for overlapping datasets first learns the individual models  and then combines them using a SAT solver \citep{triantafillou2010learning, tsamardinos2012towards, triantafillou2015constraint}. 
We believe that constraint-based approaches are preferable, specifically with cohort data, since the data is likely to have missing values and mixed variable types with non-linear relations, and constraint-based methods can accommodate these issues flexibly  \citep{witte2022multiple}.

\subsection{Overview}
In Section \ref{sec:overlap}, we give an introduction to the type of data setting that motivates our work. Section \ref{sec:eqclass} formally introduces  (tiered) background knowledge and restricted equivalence classes, while an overview and background on  the general  framework is provided in Appendix \ref{app:terminology}. In Sections \ref{sec:tfci}, \ref{sec:tiod} and \ref{sec:eff} we describe the tFCI and tIOD algorithms, and investigate their properties. In Section \ref{sec:discussion} we discuss the assumptions and limitations of the proposed approaches. We include pseudo-algorithms in Appendix \ref{app:pseudoalg} and proofs in Appendix \ref{app:proofs}.

\section{Overlapping cohort studies}\label{sec:overlap}
Our work is motivated by the wish to combine different cohort studies, i.e.\ separate temporally structured datasets,  
for causal discovery with the ultimate aim of learning  causal pathways over long time spans. It has been argued that multi-cohort designs can considerably advance life course research \citep{o2022better}.
The hope is to obtain a better understanding of e.g. how modifiable factors in early life and childhood affect health outcomes in adulthood and old age. 
An illustration  is given in in Figure \ref{fig:example:overlap}. The datasets might overlap in  different aspects: Some may have taken measurements during the same period on the same and different variables, as the international and national children cohort studies in Figure \ref{fig:example:overlap}; or they contain  measurements of the same concept (e.g. `healthy eating') at different time points, as the children's cohort studies and the adult health study with common measurements in adolescence.

\begin{figure}[!htbp]
\centering
\includegraphics{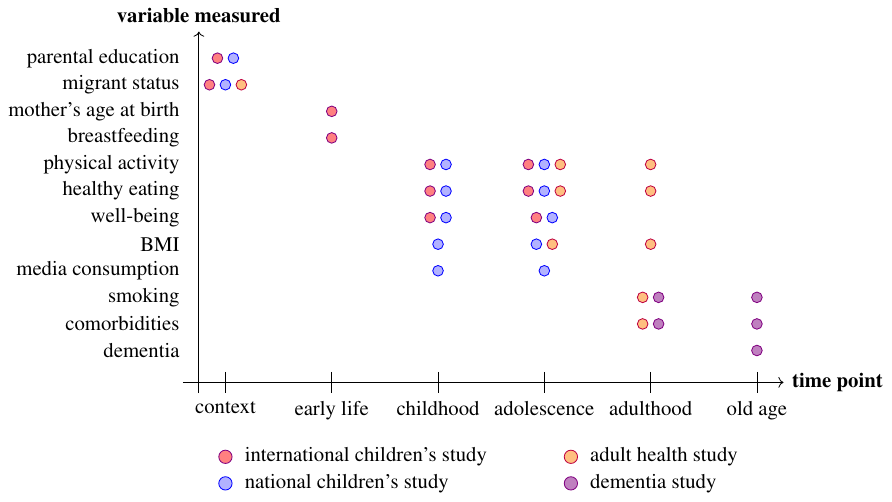}
\caption{Toy example of how four different cohort studies can overlap in time and variables.}
\label{fig:example:overlap}
\end{figure}

Thus, different cohort studies can overlap in variable sets and in time of measurements. We define a variable as being the same in two different dataset if and only if it is measured at the same time (or, here, life stage). This also means that two variables measuring the same concept, perhaps in the same dataset, are considered different if they do not refer to the same time point or life stage.
For example, consider the variable `physical activity' in Figure \ref{fig:example:overlap}. This is measured in three different datasets and at three different time points. Physical activity during childhood and adolescence in the international children's cohort study are considered two different variables, while physical activity during adolescence in the international children's cohort study, in the national cohort study, and in the adult health study are all considered the same variable.

When variables are measured in some but not all datasets, this could in principle be viewed as a missing data problem. However, a key difference is that some variables are {\em never} measured together so that there is no information on their joint distribution. In the above example, breastfeeding, BMI and dementia status are not measured jointly in any study. As pointed out by \citet{tillman2008integrating}, if we wanted to, say, impute the values of unmeasured variables the correct imputation models cannot be obtained from the measured data and would rely on correct prior knowledge of the (causal) data generating mechanism. Clearly, this is not viable in the situation where determining the causal model is the desired aim. 

While the example in Figure \ref{fig:example:overlap}  only refers to cohort studies, in practice we could have other  data structures. Although not required, our setup is most interesting in the case where all datasets  to be combined contain at least one variable that is also measured in at least one of the other datasets. 

\section{Background knowledge}\label{sec:eqclass}

The temporal structure of cohort studies induces tiered background knowledge and, in turn, restricted equivalence classes which we formalise here. 
We refer to Appendix \ref{app:terminology} for a general introduction to the relevant graphical concepts.  We  consider maximal ancestral graphs (MAGs) \citep{richardson2002ancestral}, which allow for latent variables, and partial ancestral graphs (PAGs) encoding the equivalence classes of  MAGs \citep{zhang2006causal}.

\begin{remark}
    General MAGs allow for selection bias, but here, we assume that the data contains no selection bias. Thus, we assume that there is an underlying true MAG  and that this MAG does not contain selection bias. This means that in the FCI algorithm, and consequently also in the IOD algorithm, the orientation rules R5-R7 can be omitted \citep{zhang2008completeness}. 
\end{remark}

We define \emph{background knowledge} $\K=(\R,\F)$ relative to a node set $\V$ as a set of \emph{required edges} $\R$ and a set of \emph{forbidden edges} $\F$. A graph $\G=(\V,\E)$ \emph{encodes}  $\K$ if $\E$ contains all edges in $\R$ and no edges in $\F$. A PAG $\p$ is \emph{consistent} with $\K$ if there is a MAG $\M\in [\p]$ encoding $\K$. We define the addition of background knowledge to a consistent PAG by orienting edges as in Algorithm \ref{alg:pk}. The output $\p_\K$ is a partial mixed graph (PMG) which may or may not be ancestral, and may or may not represent a restricted equivalence class. Throughout we  assume the background knowledge to be consistent with some true underlying MAG, and in that sense \emph{correct}.

\LinesNumbered
\SetKwInOut{Input}{input} 
\SetKwInOut{Output}{output}

\begin{algorithm}[!htbp]

\caption{Constructing $\p_\K$}
\label{alg:pk}
\Input{PAG $\p=(\V,\E)$ and consistent background knowledge $\K=(\R,\F)$.} 

\Output{PMG $\p_\K=(\V,\E')$}

$\E'=\E$

\ForAll{$\{ V_i\rightpartany V_j\}\in\E$}{

\uIf{$\{ V_i\rightarrow V_j\}\in\F$}{

replace $\{ V_i\rightpartany V_j\}$ with $\{ V_i\leftany V_j\}$ in $\E'$

\uElseIf{$\{ V_i\rightarrow V_j\}\in\R$}{

replace $\{ V_i\rightpartany V_j\}$ with $\{ V_i\rightarrow V_j\}$ in $\E'$

}

}

}

\end{algorithm}

Tiered background knowledge arises from partitioning the variables through a \emph{tiered ordering}, such that variables in later tiers cannot cause variables in earlier tiers.

\begin{definition}[Tiered ordering]
Let $\V$ be a node set of size $p$, and let $T\in\mathbb{N}$, $T \leq p$. A tiered ordering of the nodes in $\V$ is a map $\tau: \V\to \{ 1,\ldots ,T\}$ that assigns every node $V\in\V$ to a unique tier $t\in\{ 1,\ldots ,T\}$.
\end{definition}

\begin{definition}[Tiered background knowledge]
Let $\tau$ be a tiered ordering of the node set $\V$,  the corresponding tiered background knowledge $\K_{\tau}=(\R_{\tau},\F_{\tau})$  is then defined by  $\R_{\tau}=\emptyset$ and  $\F_{\tau}=\{ V_i \leftarrow V_j\mid \tau(V_i)<\tau (V_j), V_i,V_j\in\V \}$.
\end{definition}

If tiered background knowledge is consistent with a graph, then we say that the corresponding tiered ordering is consistent with the graph.

In Algorithm \ref{alg:pk} (line 4), a  tiered ordering $\tau$ with $\tau(A)<\tau(B)$  means that an edge $A\leftpartany B$ is oriented as $A\rightany B$ by ruling out $A\leftarrow B$. We have knowledge of this orientation because the edge connects nodes from two different tiers, and we refer to this type of edges as \emph{cross-tier edges}:

\begin{definition}[Cross-tier edge]
    Let $\G=(\V,\E)$ be a graph and let $\tau$ be a tiered ordering of the nodes in $\V$. Then, if $\tau (A)<\tau (B)$ the edge $\{A\any B\}\in\E$ is referred to as a cross-tier edge.
\end{definition}

Importantly, tiered orderings not only inform us of possible edge directions in the form of a set of forbidden edges, they also imply certain restrictions on the separations.

\begin{definition}
    Let $\G=(\V,\E)$ be a graph and $\tau$ a tiered ordering of $\V$. Then for some $A\in\V$ we define the past of $A$ in $\V$ relative to $\tau$ as $\mathrm{past}_\V^\tau(A)=\{V\in\V\mid\tau(V)\leq\tau(A)\}$. For some $\mathbf{A}\subseteq\V$, where $\mathbf{A}=\{A_1,\ldots,A_n\}$, we define the past of $\mathbf{A}$ in $\V$ relative to $\tau$ as $\mathrm{past}_\V^\tau(\mathbf{A})=\{V\in\V\mid\tau(V)\leq\max(\tau(A_1),\ldots,\tau(A_n))\}$.
\end{definition}

The next proposition follows from well-known results on the separation properties of DAGs and MAGs, but we state it specifically for the context of tiered background knowledge.

\begin{proposition}\label{prop:past}
    Let $\G=(\V,\E)$ be a DAG or MAG, $\tau$ a tiered ordering of $\V$ consistent with $\G$, and let $A,B\in\V$ be two distinct nodes. Then $A$ and $B$ are separated in $\G$ by some subset of $\V$ if and only if they are separated by a set $\mathbf{S}\subseteq\mathrm{past}_\V^\tau(A)$ or $\mathbf{S}'\subseteq\mathrm{past}_\V^\tau(B)$.
\end{proposition}

\noindent The proof of Proposition \ref{prop:past} can be found in Appendix \ref{app:proofs}.

Proposition \ref{prop:past} implies that when checking for conditional independence, we can restrict the potential separating sets to those belonging to the (joint) past. For node pairs in later tiers, this may not have much impact, but for early node pairs, this may reduce the potential separating sets substantially.

\section{The tFCI algorithm}\label{sec:tfci}

For a single dataset and allowing for latent variables, the FCI algorithm \cite{spirtes1999fci} constructs a PAG based on conditional independencies. Note that here and in the following sections we assume faithfulness (cf.\ Appendix \ref{app:terminology}). In brief, the FCI starts with a complete undirected graph, and removes edges between pairs of nodes whenever the corresponding random variables are found to be (conditionally) independent. Afterwards, it orients v-structures using the conditional independencies learned in the previous phase, and finally it applies orientation rules R1-R10 from \citet{zhang2008completeness} (see Figures \ref{fig:fcirules} and \ref{fig:zhangrules} in Appendix \ref{app:terminology}) to obtain a maximally informative PAG.

Tiered orderings are relevant to two aspects of the algorithm: For the conditional independence tests and for edge orientations. When restricting the tests performed by exploiting information from the tiered ordering via Proposition \ref{prop:past}, we obtain the \emph{simple} tFCI algorithm. Here, the objective is still to learn a PAG representing the given independence model without restriction by the background knowledge. However, we show (in Proposition \ref{prop:simpletfci} below) that the reduction of conditional independence tests still results in an algorithm that  is sound and complete, i.e.\ that the oracle version outputs the same PAG as the original FCI. This algorithm plays a central role, since the following algorithms build on it.

If we additionally orient arrowheads based on the tiered background knowledge (Algorithm \ref{alg:pk} before applying rules R1-R4 and R8-R10), then we obtain the (full) tFCI algorithm. In principle, the objective here is to recover the maximally informative graph that encodes the given independence model and the tiered background knowledge. However, the full tFCI outputs a PMG that may or may not be maximally informative (i.e. the algorithm is not complete). It does represent a superset of the MAGs in the restricted equivalence class, i.e. the algorithm is sound (cf.\ Proposition \ref{prop:fulltfci} below). We  provide the tFCI as a pseudo-algorithm (Algorithm \ref{alg:tfci}) in Appendix \ref{app:pseudoalg}, where the simple version is given by omitting lines 30-34.

Versions of the tFCI algorithm have previously been implemented and applied \citep{scheines1998, tfci, lee2022causal, petersen2023}. In fact, it is closely related to the SVAR-FCI by \citet{malinsky2018causal}. While these authors probably took for granted that the tFCI is sound, we are not aware of any formal proofs, so that we provide them here. This also leads up to the corresponding properties for the tIOD algorithm below.

\begin{proposition}\label{prop:simpletfci}
The oracle version of the simple tFCI algorithm (Algorithm \ref{alg:tfci} without lines 30-34) is sound and complete.
\end{proposition}

\begin{proposition}\label{prop:fulltfci}
The oracle version of the tFCI algorithm (Algorithm \ref{alg:tfci}) is sound.
\end{proposition}

\noindent The proofs of Propositions \ref{prop:simpletfci} and \ref{prop:fulltfci} can be found in Appendix \ref{app:proofs}. The proof strategies are analogous to those in \citet{spirtes1999fci} and \cite{colombo2012learning} with additional use of Proposition \ref{prop:past} and assuming correct  tiered background knowledge.

While the simple tFCI carries out fewer conditional independence tests, in the oracle case its output  
is identical to the one of the FCI and hence is a PAG (we comment on the sample version of the simple tFCI in Section \ref{sec:eff} and the discussion).
However, it is important to note that, in contrast to the simple tFCI, the full tFCI algorithm might not be complete, and the output might not be a maximally informative PAG. This is because with the addition of tiered background knowledge we might need more orientation rules than R1-R4 and R8-R10. Completeness of the FCI under certain types of background knowledge has been shown, e.g., local background knowledge \citep{wang2022sound}, context variables \citep{mooij2020joint}, or a restricted kind of tiered background knowledge \citep{andrews2020}. But for the general case, it is still an open question whether all orientation rules needed to ensure completeness under tiered background knowledge have been found \citep{venkateswaran2024towards}.

\section{The tIOD algorithm}\label{sec:tiod}

Let us first review the basic IOD algorithm \citep{tillman2011learning} without background knowledge. It takes as input $n$ datasets $\V_1,\ldots,\V_n$ of which any pair may share some variables\footnote{The extreme cases are that all datasets share all variables (which can then be tackled with the FCI for a unique dataset) or that the datasets have no variables in common, in which case the output set of PAGs will be very large.}. The underlying assumption of the IOD algorithm is that there is a MAG $\M=(\V,\E)$ on the full set of variables $\V=\V_1\cup\ldots\cup\V_n$
of which the individual datasets represent the marginals corresponding to their respective subset of variables. 
The IOD algorithm  returns a {\em set} of PAGs for which the marginalised independence models onto the nodes in $\V_1,\ldots, \V_n$ are equal to the independence models learned from the individual datasets. Furthermore, this set is guaranteed (in the oracle case) to contain the PAG of $\M$. The IOD algorithm consists of two parts: The first part learns the independence models from the individual datasets, and uses these models to construct a graph $\G$, which contains a superset of the adjacencies of $\M$, and a subset of the v-structures of $\M$. An individual independence model is necessarily marginal over those variables not included in this dataset and, hence, must allow for latent variables.   
The second part of the IOD algorithm constructs a graph for each possible combination of edge removals or v-structure orientations  encoding an independence model that can be marginalised into all of the $n$ independence models learned from data. We illustrate the basic IOD algorithm with the following example.

\begin{figure}[!htbp]
\centering
\includegraphics{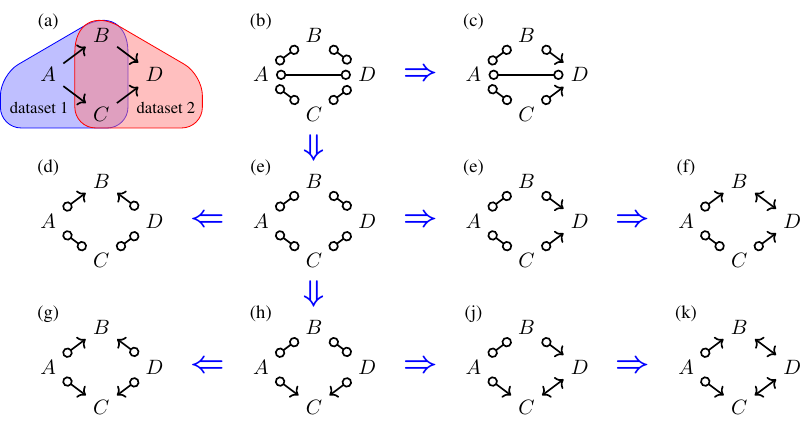}
\caption{Examples of graphs visited by the IOD algorithm. Here, (e) is the PAG of (a), but more graphs may encode the marginal independence models learned from dataset 1 and dataset 2. Given oracle knowledge of (a) as input, the algorithm considers all graphs where combinations of the edges $A\nondir D$, $A\nondir B$, $A\nondir C$, $B\nondir D$ and $C\nondir D$ are removed. Here, we illustrate the removal of $A\nondir D$. In total, the IOD visits 73 graphs, including graphs based on (b), (c),$\ldots$, (k), and it outputs the eight graphs in Figure 3.}
\label{fig:iod}
\end{figure}

\begin{example}[The IOD algorithm]\label{ex:iod}
Figure \ref{fig:iod} (a) shows a MAG over the nodes $\{A, B, C, D\}$. Here, $X_A$, $X_B$ and $X_C$ have been measured in dataset 1, and $X_B$, $X_C$ and $X_D$ in dataset 2. When learning the causal structure of $\{A, B, C, D\}$ given dataset 1 and dataset 2, the IOD algorithm proceeds as follows: First it constructs a PMG over $\{A,B,C\}$ and one over $\{B, C,D\}$, and while doing so it also constructs a common graph $\G$. At first, $\G$ is a complete non-directed graph over $\{A, B, C, D\}$. Then, for every pair of variables that are conditionally independent in dataset 1 or dataset 2, the edge between the corresponding nodes in $\G$ is removed, and we obtain the graph in Figure \ref{fig:iod} (b). At this stage of the IOD algorithm, nodes are adjacent if either they cannot be separated in the marginal graph over $\{A,B,C\}$ or $\{B, C,D\}$, or if they  are never measured jointly in the same dataset. Hence, $\G$ contains a superset of the edges of the true underlying graph. Next, the algorithm considers all edges that might not be contained in the true graph, either because the variables were never measured jointly, or because all variables necessary for obtaining conditional independence were never measured jointly. In this example, all edges are candidates for this, and the algorithm considers each combination of edge removals (32 candidate skeletons in total). Consider the case of the removal of edge $A\nondir D$: The algorithm proceeds with two possible graphs, Figure \ref{fig:iod} (b) and (e). Now, the IOD algorithm considers all possible v-structures. A v-structure might be missing in Figure \ref{fig:iod} (b) and (e) if for every dataset either the collider is not measured, or there is no separating set measured in the dataset. The triple $\langle B, D, C\rangle$ might be a v-structure, since it cannot be tested whether $X_B$ and $X_C$ are conditionally independent given $X_A$ and $X_D$ -- all we know is that they are conditionally independent given $X_A$ alone, but dependent give $X_D$ alone, which at this stage is not sufficient to determine whether $\langle B, D, C\rangle$ is a v-structure. Moreover, if $A$ and $D$ are not adjacent, then $\langle A, B, D\rangle$ or $\langle A, C, D\rangle$ might be v-structures as well: We cannot test this since $X_A$ and $X_D$ are never measured jointly. Given the skeleton in Figure \ref{fig:iod} (b), the only possible additional v-structure is $\langle B, D, C\rangle$ (Figure \ref{fig:iod} (c)), but the skeleton in Figure \ref{fig:iod} (e) allows for all combinations of $\langle B, D, C\rangle$, $\langle A, B, D\rangle$ and $\langle A, C, D\rangle$ (Figure 2 (d) to (k)), so the algorithm proceeds with all these possible graphs. The algorithm further proceeds by orienting edges according to the orientation rules R1--R10, constructing additional graphs when multiple discriminating paths are consistent with the information (see \citet{tillman2011learning} for details). Finally, it verifies that from all constructed graphs, we can construct MAGs that are consistent with the marginal independence models. It then outputs this set of PAGs. In the present example, the IOD algorithm  visits a total of 73 graphs and outputs 8 PAGs. See Table \ref{tab:ex1} in Appendix \ref{sec:example} for an overview, and Figure \ref{fig:tiod} in Example \ref{ex:tiod} for a depiction of the graphical output.
\end{example}

We extend the IOD algorithm to incorporate tiered background knowledge in multiple ways. First, we use it to restrict the conditional independence tests in the first part of the algorithm, analogously to the simple tFCI. Second, we use the tiered background knowledge for additional restrictions involving adjacencies and  v-structures so as to exclude some PAGs that are inconsistent with the background knowledge. With these two steps, the objective is to learn the set of all PAGs, and only those PAGs, that are compatible with the marginal independence models and consistent with the tiered background knowledge.  As we will see, tiered background knowledge reduces the number of potential edge removals, and the algorithm thus visits fewer potential skeletons. Moreover, it also increases the number of identifiable v-structures compared to no background knowledge. We refer to the combination of these two modifications as the {\em simple tIOD} algorithm (Algorithm \ref{alg:tiod} without lines 67-72).

In addition, we may also orient cross-tier edges and further edges as consequences of the former due to rules R1-R4 and R8-R10 (Figures \ref{fig:fcirules} and \ref{fig:zhangrules}). The algorithm using all three modifications is referred to as the (full) tIOD algorithm (Algorithm \ref{alg:tiod}). Here, the desired objective  is to recover all maximally informative graphs that are compatible with the marginal independence models and encode the tiered background knowledge. However, for the same reasons as the tFCI, the algorithm is only sound, and not necessarily complete: The tIOD outputs a set of PMGs that might not represent restricted equivalence classes. However, the algorithm is guaranteed to output at least one PMG that represents a set of MAGs containing the true underlying MAG (Proposition \ref{prop:fulltiod}).

In summary, we have the following modifications and results.

\paragraph{Modifications defining the simple tIOD algorithm (Algorithm \ref{alg:tiod})}

\begin{itemize}
    \item[(i)] When testing conditional independence in the first round, only consider separating sets that belong to the past (lines 3-13).
    \item[(ii)] Only condition on sets of nodes that belong to the past when testing for v-structures, and orient cross-tier edges that constitute v-structures (lines 14-18 and 27-31). 
    \item[(iii)] Only consider possible separating sets that belong to the past when constructing the final skeleton and v-structures (lines 19-26).
    \item[(iv)] When considering removable edges, restrict possible separating sets to the past (lines 33-37).
    \item[(v)] Orient v-structures consisting of cross-tier edges before considering possible v-structures (lines 39-47).
    \item[(vi)] Discard every PAG that is not consistent with the tiered background knowledge (line 65).
\end{itemize}

\paragraph{Additional modifications defining the full tIOD algorithm (Algorithm \ref{alg:tiod})}

\begin{itemize}
    \item[(vi)] For each PAG in the final output, orient any remaining cross-tier edges and apply orientation rules R1-R4 and R8-R10 (lines 67-72).
\end{itemize}

The  simple and full tIOD algorithms have the following properties.

\begin{proposition}\label{prop:simpletiod}
The oracle version of the simple tIOD algorithm (Algorithm \ref{alg:tiod} without lines 67-72) is sound and complete for the set of PAGs consistent with the tiered background knowledge.
\end{proposition}

\begin{proposition}\label{prop:fulltiod}
The oracle version of the full tIOD algorithm (Algorithm \ref{alg:tiod}) is sound.
\end{proposition}

\noindent The proofs are given in Appendix \ref{app:proofs}. The proof strategies are analogous to those in \citet{tillman2011learning}, but make use of Proposition \ref{prop:past}, and exploit the  tiered background knowledge which is assumed to be correct.

\section{Efficiency and informativeness}\label{sec:eff}

Tiered background knowledge leads to more informative outputs whenever the tiered ordering implies the orientation of cross-tier edges that are not involved in v-structures; this may even imply additional informativeness due to the orientation rules. Thus, the output contains more information than could be obtained from the independence model. In particular, the tFCI and tIOD algorithms then output a set of graphs that contain more information than the independence models alone. This is certain for the oracle case and can be expected for the finite sample case. 

Without the orientation of cross-tier edges, the advantages might not be as obvious. The simple tFCI algorithm outputs the same PAG as the FCI algorithm under oracle knowledge. The benefit of the simple tFCI is found, instead, in the sample case, where the output is prone to statistical errors. By skipping unnecessary statistical tests, the output is more robust as was found for the tiered PC algorithm, i.e. in the causally sufficient case \citep{bang2024improving}. Analogously, we expect the simple tiered algorithms to be more robust against statistical errors than the original FCI/IOD algorithms.

However, the simple tIOD algorithm has an interesting additional benefit. Even the simple modifications alone, without orienting cross-tier edges,  outputs a set of PAGs that should often be more informative than with the original IOD algorithm in the sense that the set contains fewer PAGs, and, hence, represents fewer possible equivalence classes. This is illustrated in Examples \ref{ex:tiod} and \ref{ex:info}.
The reason that the simple tIOD algorithm might output fewer PAGs than the IOD algorithm is that it can usually reduce the set of possible edge removals and possible v-structure orientations based on the tiered ordering. In this case, the simple tIOD algorithm is more efficient than the IOD algorithm as it needs to visit fewer graphs, cf.\ Example \ref{ex:tiod}. Here, by efficiency we refer to the algorithm visiting fewer graphs. Below we provide a proposition that states under what conditions the simple tIOD algorithm is more efficient than the IOD.

Proposition \ref{prop:informativeness} considers sets $\possdsepp{\G}{A,B}$ and $\possdsepp{\G}{B,A}$ which are sets of nodes in $\G$ containing a set that m-separates $A$ and $B$ if and only if any subset of the nodes in $\G$ m-separates $A$ and $B$ \citep{spirtes1999fci, colombo2012learning}. These sets are used by the algorithms to determine the independence models. See Appendix \ref{app:terminology} for a formal definition.

\begin{proposition}\label{prop:informativeness}
Consider the simple tIOD algorithm over $\V=\V_1\cup\ldots\cup\V_n$ using $\tau$, where $\tau: V\rightarrow \{1,\ldots,k\}$ for $k>1$. Let $\G$ be the graph obtained in line 32. Then the oracle version of the simple tIOD algorithm  visits fewer graphs than the oracle version of the IOD algorithm over $\V=\V_1\cup\ldots\cup\V_n$, if and only if one of the two following conditions holds:

\begin{itemize}
\item[(i)] 
Consider a pair of nodes $A,B$ such that $A\in\adj{\G}{B}$, and for all $i\in\{1,\ldots ,n\}$ we have
$$\{A\}\cup\adj{\G}{A}\cup\possdsepp{\G}{A,B}\cup\possdsepp{\G}{B,A}\not\subset \V_i\,\,\text{ and}$$
$$\{B\}\cup\adj{\G}{B}\cup\possdsepp{\G}{A,B}\cup\possdsepp{\G}{B,A}\not\subset \V_i.$$
For at least one such pair we  require that $\exists \, i$ such that $A,B\in\V_i$ and 
$$\left(\possdsepp{\G}{A, B}\cup\possdsepp{\G}{B,A}\right)\backslash \V_i\subseteq\V\backslash\mathrm{past}^\tau_\V(A,B)\quad\text{and}
$$
$$\adj{\G}{A}\backslash\V_i\subseteq\V\backslash\mathrm{past}_\V^\tau(A)\,\,\text{ or }\,\, \adj{\G}{B}\backslash\V_i\subseteq\V\backslash\mathrm{past}_\V^\tau(B).$$

\item[(ii)] Consider a triple of nodes $\langle A,C,B\rangle$ such that (a) $A,B\in\adj{\G}{C}$, (b) $A\notin\adj{\G}{B}$, and (c) for all $i\in\{1,\ldots n\}$ either $C\notin\V_i$ or $\mathrm{sepset}(\G_i, \{A, B\})$ is undefined. 
For at least one such triple we  require that $\tau(C)>\max(\tau(A),\tau(B))$.

\end{itemize}

\end{proposition}

\noindent The proof of Proposition \ref{prop:informativeness} can be found in Appendix \ref{app:proofs}. The proof is based on the fact that the number of visited graphs depends on the number of potential edge removals and potential v-structures. These numbers decrease in the case where tiered background knowledge renders either the potential edge removals impossible, or the v-structures certain. This occurs when for some node pairs all possible conditioning sets that have not been measured in the same dataset lie in the future (condition (i)), or if some v-structures are implied by the tiered ordering (condition (ii)).

\begin{figure}[!htbp]
\centering
\includegraphics{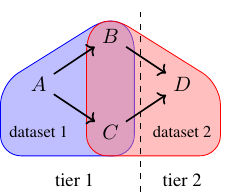}
\includegraphics{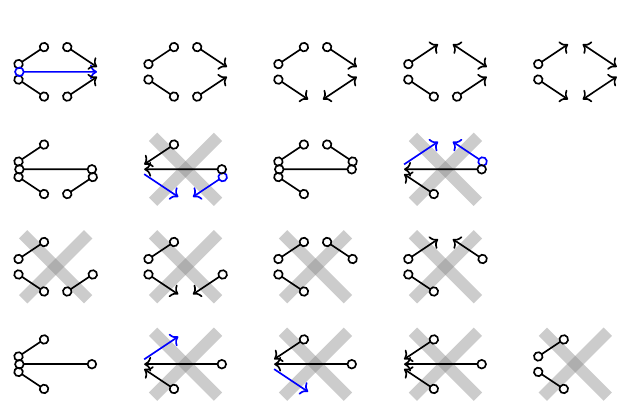}
\caption{Left: A MAG $\M$ with tiered ordering $\tau$, where the variables are measured in two datasets, dataset 1 and dataset 2. Right: All graphs visited by the tIOD algorithm (Algorithm \ref{alg:tiod}). Black edges are obtained up to line 50, blue edges are obtained from orientation rules R1-R4 and R8-R10. Crossed out graphs do not satisfy the criteria of line 65; all other graphs are output by the tIOD algorithm. In this example, the IOD algorithm visits more graphs but still outputs the same PAGs as the simple tIOD.}
\label{fig:tiod}
\end{figure}

\begin{remark}
    Let $\mathbfcal{P}$ be a set of PAGs over $\V$ obtained by the oracle version of the IOD algorithm, and $\mathbfcal{P}_\tau$ a set of PAGs obtained by the oracle version of the simple tIOD algorithm. It then follows from Proposition \ref{prop:informativeness}  that $\mathbfcal{P}_\tau$ is a subset of $\mathbfcal{P}$ only if at least  one  of the  conditions holds.
\end{remark}

Note that the output of the tFCI always has the same skeleton and at most as many circle marks as the output of the FCI. Hence, it is always  at least as informative as the FCI. Moreover, the simple tIOD also always outputs a (weak) subset of the output of IOD, thus it is  at least as informative.

\begin{remark}
Since the simple and full tIOD are identical up until line 66, the full tIOD will consider the same potential skeletons and v-structures as the simple tIOD.
\end{remark}

\begin{example}\label{ex:tiod}
Consider the MAG $\M$ of Figure \ref{fig:tiod} (a), which is the same as in Example \ref{ex:iod} but  with a tiered ordering, such that nodes $A,B$ and $C$ belong to tier 1, and $D$ belongs to tier 2. Then Proposition \ref{prop:past} implies that the only candidate separating set for $A$ and $B$ is $C$ and the only candidate separating set for $A$ and $C$ is $B$. Since all of the corresponding variables are measured in the same dataset, we see that these are in fact not separating sets, and it is not necessary to consider skeletons without $A\nondir B$ and $A\nondir C$ in the tIOD algorithm. The tIOD algorithm considers $A\nondir D$, $B\nondir D$, and $C\nondir D$ as removable edges and considers 8 different skeletons (see Table \ref{tab:ex2} in Appendix \ref{sec:example}), which is considerably fewer than visited by the IOD algorithm in Example \ref{ex:iod}. 

In dataset 1 we find that $X_B\indep X_C\mid X_A$. In dataset 2 we find no independencies, in particular $X_B\not\indep X_C\mid X_D$, and in all graphs with both edges $B\nondir D$ and $C\nondir D$ present, $\langle B, D, C\rangle$ is considered a potential v-structure by the IOD algorithm in Example \ref{ex:iod}, while the tIOD algorithm orients $\langle B, D, C\rangle$ as a v-structure because $D$ is in a later tier than $B$ and $C$. In total, the simple tIOD algorithm visits 18 graphs (see Table \ref{tab:ex2} in Appendix \ref{sec:example}) and outputs 8 graphs. 
\end{example}

\begin{example}\label{ex:info}
    Example \ref{ex:tiod} is a case where the tIOD visits fewer graphs but outputs  the same PAGs as the IOD.
    A key reason for this is that both algorithms require the output to be consistent with the learned (marginal) independence models. The only independence found was $X_B\indep X_C\mid X_A$. This implicitly requires the graph to have a path between $B$ and $C$, that can be separated by $A$ without also conditioning on $D$ and requires $\langle B, D, C\rangle$ to be a v-structure whenever it is an unshielded triple, which much reduces the number of output graphs. 
    
    A different situation is illustrated in Figure \ref{fig:info}, where nodes $A,B,D$ are measured in dataset 1, while $A,C,D$ are measured in dataset 2, with the same true MAG as before.    
    No independencies are found in dataset 1 and dataset 2. This leaves a complete graph as the intermediate graph $\G$, Figure \ref{fig:info} (b). Hence, the IOD algorithm  outputs, among others, Figure \ref{fig:info} (c). However,  this graph cannot be output by the tIOD algorithm since it  orients $\langle B, D, C\rangle$ as a v-structure in every graph where this is an unshielded triple. In this case, the tIOD algorithm is not just more efficient, it  also outputs fewer graphs and is thus  more informative. 
\end{example}

\begin{figure}
\centering
\includegraphics{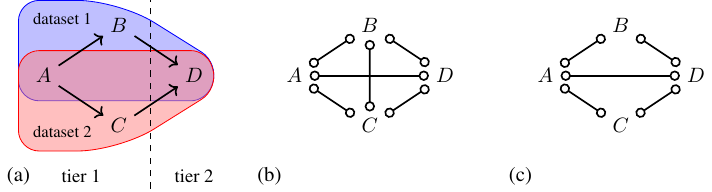}
\caption{(a) MAG with tiered ordering of the nodes, with measurements in dataset 1 and dataset 2. (b) The intermediate graph $\G$ obtained at line 32 of the tIOD algorithm (Algorithm \ref{alg:tiod}) with oracle knowledge of (a) as input. (c) An example graph output by the IOD algorithm, with (a) as input), that would not have been output by the simple tIOD.}
\label{fig:info}
\end{figure}

\section{Discussion}\label{sec:discussion}

We have introduced extensions of constraint-based causal discovery algorithms that exploit tiered background knowledge
while allowing for latent variables and multiple overlapping datasets. The simple versions of tFCI and tIOD can be seen as estimating (sets of) equivalence classes whereas the full versions make further use of the background knowledge at the expense of completeness. We have shown that the simple versions of tFCI and tIOD are sound and complete, and we have shown when exactly using the background knowledge reduces the number of graphs visited by the simple tIOD. In consequence, the simple (not just the full) tIOD outputs a set of graphs that can often be expected to be considerably more informative than without the tiered ordering. 

The results in this paper apply to the oracle case. In practice, however, conditional independence testing is subject to statistical errors. Here, the simple tFCI and tIOD are attractive as they omit unnecessary conditional independence tests by exploiting Proposition \ref{prop:past}. Indeed, for analogous reasons, tiered background knowledge has been shown to improve the statistical properties of the tiered PC algorithm \citep{bang2024improving} -- similarly robust finite sample behavior can be expected for the  sample versions of the tFCI and tIOD algorithms. A further reason for
exploiting tiered background knowledge to detect v-structures, as done by the simple tFCI and simple tIOD, is to ensure that the estimated equivalence classes are consistent with that background knowledge which is not otherwise guaranteed for finite samples. 

We have explicitly assumed one single underlying joint MAG over all variables. This means that the different datasets must inform us about comparable settings. Hence,  the marginal equivalence classes are consistent on the overlapping nodes under oracle knowledge, so that it makes sense to combine them. However, in finite samples we may encounter the case that the  independence models, if estimated separately with the different datasets, are not consistent on the overlapping variables even under this assumption. This can easily be fixed by combining the statistical tests and / or the datasets on the overlapping variables before inputting the estimated independence models into the algorithm; one such solution is given in \citep{tillman2011learning}. 

A number of further extensions could be interesting for future work. For instance, the principles of the tFCI algorithm could be applied to other versions of the FCI algorithm such as the anytime FCI \citep{spirtes2001anytime}, the RFCI \citep{colombo2012learning}, the FCI+ \citep{claassen2013learning}, the order independent or conservative FCI \citep{colombo2014}, etc. Future research might explore how the ideas of these algorithms could also be utilised in the tIOD algorithm.

Often, tiered background knowledge is not the only kind of background knowledge available. Typical  epidemiological applications, in fact, often rely on fully specified  expert causal graphs rather than data driven graph construction \citep{petersen2023constructing}. More realistically, even experts may only have partial knowledge of the causal structure \citep{didelez2024invited}.  
It would be an interesting direction for future research to explore how one could use the tIOD algorithm to combine or refine graphs based on a mix of (possibly multiple) experts  and (possibly multiple) data sources.

Our work is motivated by multi-cohort studies, which may or may not contain repeated measures, but always contain a tiered structure. Temporal structure is also available with time series data. Extensions of the FCI algorithm to time series exist \citep{malinsky2018causal, gerhardus2020high}. These are based on similar principles,  essentially  exploiting corresponding versions of Proposition \ref{prop:past}, i.e. not conditioning on the future. Additional common assumptions such as stationarity and considerations of autocorrelation are crucial for time series which we did not cover, here. While we are not aware that multiple overlapping datasets have been considered in the literature on causal analysis of time series, it is conceivable that a subvector of the multivariate time series is measured in one dataset while another subvector is measured in another dataset. It may be promising to develop a version of the tIOD algorithm for such settings.

\acks{This work was funded by the Deutsche Forschungsgemeinschaft (DFG, German Research Foundation)
– Project numbers 281474342/GRK2224/2 and 459360854. We would like to thank the reviewers for their constructive comments, which helped improve the paper.}

\bibliography{ref}

\appendix

\section{Terminology}\label{app:terminology}

\subsection{Ancestral graphs}

\paragraph{Graphs, nodes and edges} A graph $\G=(\V,\E)$ consists of a set of \textit{nodes} $\V$ and \textit{edges} $\E$. We allow edges to have three types of \textit{edge marks}: $\circ$ (\textit{circle mark}), $-$ (\textit{tail}) and $>$ (\textit{arrowhead}), and we allow the following types of edges: $\nondir$ (\textit{non-directed}), $\rightpart$ (\textit{partially directed}), $\rightarrow$ (\textit{directed}), and $\leftrightarrow$ (bidirected).  We use $\any$ (both edge marks can be of any type), $\rightany$ (left edge mark can be of any type) and $\leftpartany$ (left edge mark can a circle mark or arrowhead) as placeholders.  We say that an edge (or tail) $A\rightarrow B$ is out of $A$ and an edge (or arrowhead)  $A\rightany B$ is into $B$.  Let $A, B\in\V$ and $\{A\any B\}\in\E$, then $A$ and $B$ are \emph{adjacent} in $\G$. If all pairs of nodes in a graph are adjacent, then the graph is \emph{complete}. By $\adj{\G}{A}$ we denote all the nodes that are adjacent to $A\in\V$ in $\G$. We allow for at most one edge between any pair of nodes, and no node can be adjacent to itself. A graph that only contains directed/bidirected/non-directed edges is a directed/bidirected/non-directed graph.

\paragraph{Graphs and subgraphs} Let $\G=(\V,\E)$ and $\G'=(\V',\E')$ be distinct graphs. If $\V'\subseteq\V$ and $\E'\subseteq\E$ then $\G'$ is a \emph{subgraph} of $\G$. The \emph{non-directed subgraph} $\G_U=(\V,\E_U)$ of $\G$ is obtained by removing all edges in $\E$ that contain an arrowhead, such that $\E_U\subseteq\E$. The \emph{induced subgraph} of $\G$ over $\V'\subseteq\V$ is defined as $\G_{\V'}=(\V',\E_{\V'})$ where $\E_{\V'}\subseteq\E$ contains all edges between nodes in $\V'$. By $\G_{\E'}=(\V,\E')$ we denote the subgraph of $\G$ obtained by removing all edges not in $\E'\subseteq$. The skeleton of $\G$ is obtained by replacing every $\{A\any B\}\in\E$ with $\{A\nondir B\}$.

\paragraph{Paths and cycles} Let $\pi=\langle V_1,\ldots ,V_n\rangle$ be a sequence of adjacent nodes. If each node only occurs once in $\pi$, then this is as a path from $V_1$ to $V_n$. If for any $\langle V_{i-1}, V_i, V_{i+1}\rangle$, $1<i<n$, $V_{i-1}$ and $V_{i+1}$ are not adjacent, then this is an unshielded triple, and if every triple on $p$ is unshielded, then $\pi$ is an unshielded path. Let $\pi=\langle V_1,\ldots V_n\rangle$ be a path. If  $V_{i-1}\rightany V_i\leftany V_{i+1}$ occurs on $\pi$, then $V_i$ is a collider on $\pi$ and otherwise it is a non-collider on $\pi$. If in addition $\langle V_{i-1}, V_i, V_{i+1}\rangle$ is an unshielded triple, then this is a \emph{v-structure}. If on $\pi$ $V_i\rightarrow V_{i+1}$ for every $1\leq i <n$, then $\pi$ is a directed path (from $V_1$ to $V_n$). If on $\pi$ $V_i\nondir V_{i+1}$ for every $1\leq i <n$, then $\pi$ is a non-directed path (from $V_1$ to $V_n$). A path $\langle V_1,\ldots V_n\rangle$ is possibly directed from $V_1$ to $V_n$ if for every $1\leq i<n$, the edge between $V_i$ and $V_{i+1}$ is not into $V_i$ or out of $V_{i+1}$. When needed, we  denote a possibly directed path from $V_1$ to $V_n$ by $V_1\dashrightarrow V_n$. A (non-) directed path from $V_1$ to $V_n$ combined with a (non-) directed path from $V_n$ to $V_1$ is a (non-) directed cycle. A directed path from $V_1$ to $V_n$ combined with $V_1\leftrightarrow V_n$ is an almost directed cycle. Let $\G'=(\V,\E')$ be a graph with same skeleton as $\G=(\V,\E)$ but possibly $\E'\neq\E$, then for a path $\pi$ in $\G$ the corresponding path in $\G'$  is the path $\pi'$ that consists of the same nodes as $\pi$. A path $\pi=\langle V_1,\ldots,V_n\rangle$ is an \emph{inducing path} if every node $V_i$, $1<i<n$, is a collider on $\pi$ and an ancestor of $V_1$ or $V_n$.

\paragraph{Parents, ancestors and descendants} For $\G=(\E,\V)$ we define the \emph{parents}, \emph{ancestors}, \emph{possible ancestors}, \emph{descendants} and \emph{possible descendants} of a node $A$ in $\G$ as follows
\begin{align*}
    \pa{\G}{A}&=\{V\in\V\mid \{V\rightarrow A\}\in\E\}\\
    \an{\G}{A}&=\{V\in\V\mid \G\text{ contains a directed path from }V\text{ to } A\}\\
    \possan{\G}{A}&=\{V\in\V\mid \G\text{ contains a possibly directed path from }V\text{ to } A\}\\
    \de{\G}{A}&=\{V\in\V\mid \G\text{ contains a directed path from }A\text{ to } V\}\\
    \possde{\G}{A}&=\{V\in\V\mid \G\text{ contains a possibly directed path from }A\text{ to } V\}
\end{align*}

\noindent For a set $\mathbf{A}$ we define $\an{\G}{\mathbf{A}}=\{V\in\V\mid \G\text{ contains a directed path from }V\text{ to some } A\in\mathbf{A}\}$.

\paragraph{Graph types}  A directed graph that does not contain any directed cycles is called a directed acyclic graph (DAG). A mixed graph (MG) contains both directed and bidirected edges. A graph that contains no directed cycles or almost directed cycles is ancestral. An ancestral mixed graph that does not contain an inducing path between any pair of non-adjacent nodes is a maximal ancestral graph (MAG). A partial mixed graph (PMG) contains directed, bidirected, non-directed and partially directed edges.

\subsection{Independence models and Markov equivalence}

\begin{definition}[m-connecting]
    Let $\G=(\V,\E)$ be a MAG, $\pi=\langle V_1,\ldots ,V_n\rangle$ be a path in $\G$, and let $\mathbf{S}\subseteq\V\backslash\{V_1,V_n\}$. If 
    
    \begin{itemize}
        \item[(i)] for every collider $V$ on $\pi$, either $V$ or a descendant of $V$ is in $\mathbf{S}$, and
        \item[(ii)] no non-collider on $\pi$ is in $\mathbf{S}$
    \end{itemize}

\noindent  then $\pi$ is m-connecting given $\mathbf{S}$.
\end{definition}

For a MAG  $\G=(\V,\E)$, if there exists a path from a $A\in\V$ to $B\in\V$ where $B\neq A$ that is m-connecting given $\mathbf{S}$, we say that $A$ and $B$ are m-connected given $\mathbf{S}$. If no such path exists, we say that $A$ and $B$ are m-separated given $\mathbf{S}$ in $\G$ and denote this by $A\perp_\G B\mid\mathbf{S}$. We define an independence model $\mathcal{I}(\G)$ induced by $\G$ as the collection of all m-separations in $\G$: 
\begin{align*}
    (A\perp_\G B\mid\mathbf{S})\in\mathcal{I}(\G)\Leftrightarrow A\text{ and }B\text{ are m-separated by }\mathbf{S}\text{ in }\G 
\end{align*}

The same definitions are valid for DAGs, and here m-separation (m-connection) is equivalent to d-separation (d-connection). Occasionally, we omit the ``m'' or ``d'' and simply refer to separation in a graph.

\begin{definition}[$\mathrm{dsep}_\G(A,B)$ \citep{spirtes1999fci}]\label{def:dsep}
    Let $\G=(\V,\E)$ be a MAG and let $A,B,V$ be distinct nodes in $\V$. Then $V\in\mathrm{dsep}_\G(A,B)$ if and only if there is a path $\langle A=V_0,V_1,\ldots,V_{n+1}=V\rangle$ for which 
    \begin{itemize}
        \item[(i)] $V_i\in\an{\G}{\{A,B\}}$  for every $i\in\{1,\ldots n\}$, and
        \item[(ii)] $V_i\leftrightarrow V_{i+1}$ for every $i\in\{1,\ldots n-1\}$.
    \end{itemize}
\end{definition}

\begin{lemma}[Lemma 12 in \citet{spirtes1999fci}]\label{lemma:dsep}
    Let $\G=(\V,\E)$ be a MAG and $X,Y\in\V$. There exists a set $\V'\subseteq\V\backslash\{ A,B\}$ such that $A$ and $B$ are m-separated by $\V'$ if and only if they are m-separated by $\mathrm{dsep}_\G(A,B)$.
 \end{lemma}

Two MAGs (or DAGs) $\G_1$ and $\G_2$ are Markov equivalent if they induce the same independence model: $\mathcal{I}(\G_1)=\mathcal{I}(\G_2)$. A (Markov) equivalence class is a set of Markov equivalent MAGs (or DAGs). A partial ancestral graph (PAG) represents an equivalence class of MAGs. A PAG $\p$ represents an equivalence class of MAGs $[\p]$ such that every arrowhead and every tail present in $\p$ is also present in every $\mathcal{\M}\in[\p]$, and we assume that they are maximally informative in the sense that for every circle mark, there is at least one $\mathcal{\M}_1\in[\p]$ where this is an arrowhead, and one $\mathcal{\M}_2\in[\p]$ where this is a tail.

In addition to Markov equivalence, a class of graphs might share more information, and constitute a restricted equivalence class. In general, we can represent the additional background knowledge using a PMG, which may not be ancestral and maximally informative. Given a sufficient set of orientation rules we may represent a restricted equivalence class using a maximally informative PAG. The orientation rules needed depends on the type of background knowledge.

\subsection{The FCI algorithm}

We assume that the nodes $\V$ in a graph $\G=(\V,\E)$ represent a set of random variables $\mathbf{X}_\V$. Throughout, we assume that the \emph{global Markov property} and \emph{faithfulness} hold, which allows us to learn an equivalence class of causal graphs from data. Let $\D=(\V,\E)$ be a DAG and let $P$ be a probability distribution over $\mathbf{X}_\mathbf{V}$. If $P$ obeys the global Markov property and faithfulness with respect to $\D$ then for any distinct nodes $A,B\in\V$ and set $\mathbf{S}\subseteq\V\backslash\{ A,B\}$
    
    \begin{align}
        A\perp_{\D} B\mid \mathbf{S}\Leftrightarrow X_{A}\indep X_{B}\mid \mathbf{X}_\mathbf{S}
    \end{align}

    \noindent where $\indep$ denotes (conditional) independence between random variables.

We can still utilise the global Markov property and faithfulness without causal sufficiency, since d-separation in a DAG corresponds to m-separation in a MAG that has been constructed by marginalising over the latent variables (see \citet{richardson2002ancestral} for details).

The FCI algorithm searches for possible separating sets in the following set:

\begin{definition}[$\possdsep{\G}{A,B}$ \citep{spirtes1999fci}]
    Let $\M=(\V,\E)$ be a PMG and $A$ and $B$ distinct nodes in $\V$. Then $V\in\V$ is in $\possdsep{\G}{A,B}$ if and only if  there is a path $\pi=\langle A=V_1, \ldots , V_n = V\rangle$ between $V$ and $A$ that in $\G$ such that $\pi$ is  not directed, and for every subpath $\langle V_{j-1}, V_j, V_{j+1}\rangle$, $2\leq j\leq n-1$, of $\pi$, either 
        \begin{itemize}
            \item[(a)] $V_j$ is a collider, or
            \item[(b)] $V_j$ is a definite non-collider and $V_{j-1}$ and $V_{j+1}$ are adjacent.
        \end{itemize}
\end{definition}

\noindent This set was refined by \citet{colombo2012learning} in the following way: $$\possdsepp{\G}{A,B}=\{V\in\possdsep{\G}{A,B}\mid V\text{ is on a path between $A$ and $B$ in }\G\}$$

\begin{definition}[Discriminating path \citep{zhang2008completeness}]
    Let $\M=(\V,\E)$ be a MAG, and let $\pi=\langle A,\ldots , B, C, D\rangle$ be a path between $A\in\V$ and $D\in\V$ in $\M$. Then $\pi$ is a discriminating path for $C$ if

    \begin{itemize}
        \item[(i)] $\pi$ consists of at least three edges, and
        \item[(ii)] $C$ is adjacent to $D$ in $\pi$ and is a non-endpoint node on $\pi$, and
        \item[(iii)] $A\notin\adj{\M}{D}$, and every node between $A$ and $C$ on $\pi$ is a collider and a parent of $D$.
    \end{itemize}
\end{definition}

\begin{figure}[!htbp]
\centering
\begin{tikzpicture}[scale = 0.6]

\node (A1) at (0,2) {$A$};
\node (B1) at (0,0) {$B$};
\node (C1) at (2,0) {$C$};

\node (A1') at (4,2) {$A$};
\node (B1') at (4,0) {$B$};
\node (C1') at (6,0) {$C$};

\node (r1) at (3, 1) {$\overset{\textnormal{R1}}{\Longrightarrow}$};

\node (A2) at (0,-2) {$A$};
\node (B2) at (0,-4) {$B$};
\node (C2) at (2,-4) {$C$};

\node (A2') at (4,-2) {$A$};
\node (B2') at (4,-4) {$B$};
\node (C2') at (6,-4) {$C$};

\node (r2) at (3, -3) {$\overset{\textnormal{R2}}{\Longrightarrow}$};

\node (A2'') at (8,-2) {$A$};
\node (B2'') at (8,-4) {$B$};
\node (C2'') at (10,-4) {$C$};

\node (A2''') at (12,-2) {$A$};
\node (B2''') at (12,-4) {$B$};
\node (C2''') at (14,-4) {$C$};

\node (r2') at (11, -3) {$\overset{\textnormal{R2}}{\Longrightarrow}$};

\node (A3) at (0,-6) {$A$};
\node (B3) at (0,-8) {$B$};
\node (C3) at (2,-8) {$C$};
\node (D3) at (2, -6) {$D$};

\node (A3') at (4,-6) {$A$};
\node (B3') at (4,-8) {$B$};
\node (C3') at (6,-8) {$C$};
\node (D3') at (6, -6) {$D$};

\node (r3) at (3, -7) {$\overset{\textnormal{R3}}{\Longrightarrow}$};

\node (A4) at (0, -12) {$A$};
\node (V1) at (2, -12) {$\cdots$};
\node (V2) at (4, -12) {$\cdots$};
\node (B4) at (6, -12) {$B$};
\node (C4) at (8, -12) {$C$};
\node (D4) at (8, -10) {$D$};

\node (A4') at (10, -12) {$A$};
\node (V1') at (12, -12) {$\cdots$};
\node (V2') at (14, -12) {$\cdots$};
\node (B4') at (16, -12) {$B$};
\node (C4') at (18, -12) {$C$};
\node (D4') at (18, -10) {$D$};

\node (A4'') at (0, -16) {$A$};
\node (V1'') at (2, -16) {$\cdots$};
\node (V2'') at (4, -16) {$\cdots$};
\node (B4'') at (6, -16) {$B$};
\node (C4'') at (8, -16) {$C$};
\node (D4'') at (8, -14) {$D$};

\node (r4) at (9, -11) {$\overset{\textnormal{R4}}{\Longrightarrow}$};
\node (r4') at (4, -13) {$\Downarrow\,${\scriptsize R4}};
\node (s4) at (11.25, -10) {\small $C\in\mathrm{sepset}(A, D)$};
\node (s4') at (1.25, -14) {\small $C\notin\mathrm{sepset}(A, D)$};

\node[align=left] (t) at (16, -15) {$\langle A,\ldots ,B, C, D\rangle$ is a discriminating path for $C$};

\draw[anydir]
(A1) edge (B1)
(A1') edge (B1')
(B2) edge (C2)
(B2') edge (C2')
(A2') edge (C2')
(A2'') edge (B2'')
(A2''') edge (B2''')
(A2''') edge (C2''')
(B3) edge (C3)
(D3) edge (C3)
(B3') edge (C3')
(D3') edge (C3')
(A3') edge (C3')
(A4) edge (V1)
(A4') edge (V1')
(A4'') edge (V1'')
;

\draw[anypart]
(C1) edge (B1)
(A2) edge (C2)
(A2'') edge (C2'')
(D3) edge (A3)
(B3) edge (A3)
(D3') edge (A3')
(B3') edge (A3')
(A3) edge (C3)
(D4) edge (C4)
;

\draw[dir]
(B1') edge (C1')
(A2) edge (B2)
(A2') edge (B2')
(B2'') edge (C2'')
(B2''') edge (C2''')
(V1) edge [bend left] (D4)
(V2) edge [bend left] (D4)
(B4) edge [bend left] (D4)
(V1') edge [bend left] (D4')
(V2') edge [bend left] (D4')
(B4') edge [bend left] (D4')
(V1'') edge [bend left] (D4'')
(V2'') edge [bend left] (D4'')
(B4'') edge [bend left] (D4'')
(C4') edge (D4')
;

\draw[bidir]
(V1) edge (V2)
(V2) edge (B4)
(B4) edge (C4)
(V1') edge (V2')
(V2') edge (B4')
(B4') edge (C4')
(V1'') edge (V2'')
(V2'') edge (B4'')
(B4'') edge (C4'')
(D4'') edge (C4'')
;

\end{tikzpicture}
\caption{Orientation rules for the FCI algorithm \citep{spirtes1999fci, zhang2008completeness}. Here $\blacksquare$ is a placeholder edge mark equivalent to $*$.}
\label{fig:fcirules}
\end{figure}
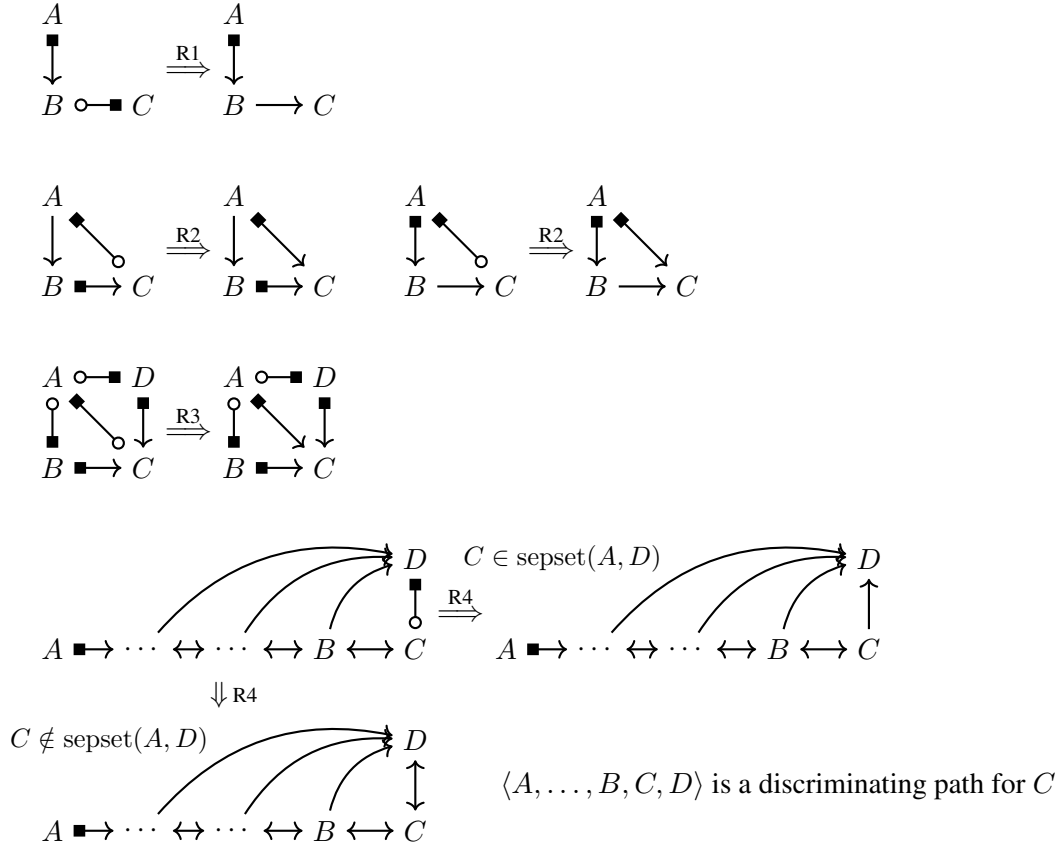

\begin{figure}
\centering
\begin{tikzpicture}[scale = 0.6]

\node (A8) at (0,2) {$A$};
\node (B8) at (0,0) {$B$};
\node (C8) at (2,0) {$C$};

\node (A8') at (4,2) {$A$};
\node (B8') at (4,0) {$B$};
\node (C8') at (6,0) {$C$};

\node (r8) at (3, 1) {$\overset{\textnormal{R8}}{\Longrightarrow}$};

\node (A9) at (0,-2) {$A$};
\node (B9) at (0,-4) {$B$};
\node (C9) at (2,-4) {$C$};
\node (D9) at (2,-2) {$D$};

\node (A9') at (4,-2) {$A$};
\node (B9') at (4,-4) {$B$};
\node (C9') at (6,-4) {$C$};
\node (D9') at (6,-2) {$D$};

\node (r9) at (3, -3) {$\overset{\textnormal{R9}}{\Longrightarrow}$};

\node[align=left] (t1) at (13, -3)  {$\langle B, \ldots, A,\ldots , D,\ldots ,C\rangle$ is a possibly directed\\ unshielded path from $B$ to $C$};

\node (r10) at (5, -7) {$\overset{\textnormal{R10}}{\Longrightarrow}$};

\node (V1) at (0,-6) {$V_1$};
\node (B10) at (0,-8) {$B$};
\node (A10) at (2,-6) {$A$};
\node (C10) at (2,-8) {$C$};
\node (D10) at (4,-8) {$D$};
\node (V2) at (4,-6) {$V_2$};

\node (V1') at (6,-6) {$V_1$};
\node (B10') at (6,-8) {$B$};
\node (A10') at (8,-6) {$A$};
\node (C10') at (8,-8) {$C$};
\node (D10') at (10,-8) {$D$};
\node (V2') at (10,-6) {$V_2$};

\node[align=left] (t2) at (16.5, -7) {$\langle A,\ldots ,B\rangle$ and $\langle A,\ldots ,D\rangle$ are unshielded,\\ $V_1\in\adj{}{A}$ (perhaps $V_1=B$) and\\ $V_2\in\adj{}{A}$ (perhaps $V_2=D$), and\\ $V_1\neq V_2$ and $V_1\notin\adj{}{V_2}$};

\draw[dir]
(A8) edge (B8)
(B8) edge (C8)
(A8') edge (B8')
(B8') edge (C8')
(A8') edge (C8')
(B9') edge (C9')
(B9) edge [dashed] (A9)
(A9) edge [dashed] (D9)
(D9) edge [dashed] (C9)
(B9') edge [dashed] (A9')
(A9') edge [dashed] (D9')
(D9') edge [dashed] (C9')

(A10) edge [dashed] (V1)
(A10) edge [dashed] (V2)
(V1) edge [dashed] (B10)
(V2) edge [dashed] (D10)
(B10) edge (C10)
(D10) edge (C10)
(A10') edge [dashed] (V1')
(A10') edge [dashed] (V2')
(V1') edge [dashed] (B10')
(V2') edge [dashed] (D10')
(B10') edge (C10')
(D10') edge (C10')
(A10') edge (C10')
;

\draw[partdir]
(A8) edge (C8)
(B9) edge (C9)
(A10) edge (C10)
;
    
\end{tikzpicture}
\caption{Orientation rules for the FCI algorithm \citep{zhang2008completeness}}
\label{fig:zhangrules}
\end{figure}
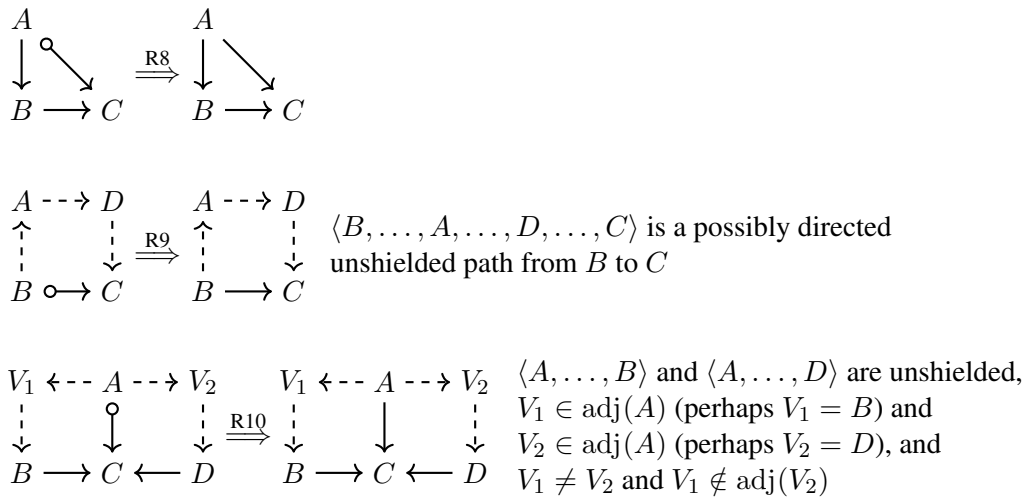

\clearpage

\section{Pseudo-algorithms}\label{app:pseudoalg}

By $\mathscr{P}(\cdot)$ we denote the power set. By $X \indep Y\mid Z$ we indicate conditional independence between $X$ and $Y$ given $Z$. All independence facts are obtained through oracle knowledge.

\setcounter{AlgoLine}{0}

\begin{algorithm}[!htbp]

\caption{The tiered FCI (tFCI) algorithm}
\label{alg:tfci}
\Input{Node set $\V$, tiered ordering $\tau$, and an independence oracle over $\X_\V$} 

\Output{partial mixed graph $\G=(\V,\E)$}

Let $\G=(\V,\E)$ be a complete graph with non-directed edges.

$k=0$

\Repeat{there is no ordered pair of adjacent edges $V_i$ and $V_j$ with $|\adj{\G}{V_i}\cap\mathrm{past}^\tau_\V(V_i)\backslash \{V_j\}|\geq k$}{

\For{each ordered pair of adjacent nodes $V_i$ and $V_j$}{

\If{there exists a $\mathbf{S}\subseteq\adj{\G}{V_i}\cap\mathrm{past}^\tau_\V(V_i)\backslash \{V_j\}$ of size $k$ with $X_{V_i}\indep X_{V_j}\mid\mathbf{X}_\mathbf{S}$}{

remove $\{V_i\nondir V_j\}$ from $\E$

add $\mathbf{S}$ to $\mathrm{sepset}(V_i,V_j)$

} 

} 

$k = k + 1$

} 

\For{each unshielded triple $\langle V_i, V_j, V_k\rangle$}{

\If{$\tau(V_j)>\max(\tau(V_i),\tau(V_k))$ or $V_j\notin\mathrm{sepset}(V_i,V_k)$}{

orient $V_i\any V_j\any V_k$ as $V_i\rightany V_j\leftany V_k$

}

} 

\For{each ordered pair of adjacent nodes $V_i$ and $V_j$}{

\If{there exists a $\mathbf{S}\subseteq\possdsepp{\G}{V_i, V_j}\cap\mathrm{past}^\tau_\V(\{V_i,V_j\})\backslash \{V_j\}$ of size $k$ with $X_{V_i}\indep X_{V_j}\mid\mathbf{X}_\mathbf{S}$}{

remove $\{V_i\nondir V_j\}$ from $\E$

add $\mathbf{S}$ to $\mathrm{sepset}(V_i,V_j)$

} 

} 

replace each edge $\{V_i\any V_j\}$ in $\E$ with $\{V_i\nondir V_j\}$

\For{each unshielded triple $\langle V_i, V_j, V_k\rangle$ }{

\If{$\tau(V_j)>\max(\tau(V_i),\tau(V_k))$ or  $V_j\notin\mathrm{sepset}(V_i,V_k)$}{

orient $V_i\any V_j\any V_k$ as $V_i\rightany V_j\leftany V_k$

} 

} 

\texttt{\#omit lines 30-34 in the simple tFCI algorithm}

\For{each ordered pair of adjacent nodes $V_i, V_j$}{

\If{$\tau(V_i)<\tau(V_j)$}{

orient $V_i\leftpartany V_j$ as $V_i\rightany V_j$

} 

} 

apply orientation rules R1-R4 (Figure \ref{fig:fcirules}) and R8-R10 (Figure \ref{fig:zhangrules}) repeatedly to $\G$ until none applies

\end{algorithm}

\setcounter{AlgoLine}{0}

\begin{algorithm}[!htbp]
\small
\caption{The tiered IOD (tIOD) algorithm}

\Input{Node sets $\V_1,\ldots, \V_n$, tiered ordering $\tau$, and independence oracles over $\X_{\V_1},\ldots,\X_{\V_n}$} 

\Output{set of partial mixed graphs $\mathbfcal{G}$}

Let $\G=(\V,\E)$ be a complete graph with non-directed edges.

\For{each $l\in\{1,\ldots ,n\}$}{

Let $\G_l=(\V_l,\E_l)$ be a complete graph with non-directed edges.

$k = 0$

\Repeat{there is no ordered pair of adjacent edges $V_i$ and $V_j$ with $|\adj{\G_l}{V_i}\cap\mathrm{past}^\tau_{\V_l}(V_i)\backslash \{V_j\}|\geq k$}{

\For{each ordered pair of adjacent nodes $V_i$ and $V_j$}{

\If{there exists a set $\mathbf{S}\subseteq\adj{\G_l}{V_i}\cap\mathrm{past}^\tau_{\V_l}(V_i)$ of size $k$ with $X_{V_i}\indep X_{V_j}\mid \X_\mathbf{S}$}{
remove $\{V_i\nondir V_j\}$ from $\E_l$ and from $\E$

add $\mathbf{S}$ to $\mathrm{sepset}(\G_l, \{V_i,V_j\})$
} 

} 

$k= k+1$
} 

\For{each unshielded triple $\langle V_i, V_j, V_k\rangle$ }{

\If{$\tau(V_j) >\max(\tau(V_i),\tau(V_k))$ or $V_j\notin\mathrm{sepset}(\G_l, \{V_i,V_k\})$}{

orient $V_i\any V_j\any V_k$ as $V_i\rightany V_j \leftany V_k$
} 

} 

\For{each ordered pair of adjacent nodes $V_i$ and $V_j$}{

\uIf{there exists a set $\mathbf{S}\subseteq\possdsepp{\G_l}{V_i, V_j}\cap\mathrm{past}^\tau_{\V_l}(\{V_i, V_j\})$  with $X_{V_i}\indep X_{V_j}\mid \X_\mathbf{S}$}{
remove $\{V_i\nondir V_j\}$ from $\E_l$ and from $\E$

add $\mathbf{S}$ to $\mathrm{sepset}(\G_l, \{V_i,V_j\})$
} 
\Else{

add $\langle \{ V_i,V_k\},\V_l\rangle$ to \textbf{InducingPaths}
} 

} 

\For{each unshielded triple $\langle V_i, V_j, V_k\rangle$ in $\G_l$}{
\If{$V_j\in\adj{\G}{V_i}\cup\adj{\G}{V_k}$, and $\tau(V_j) >\max(\tau(V_i),\tau(V_k))$ or $V_j\notin\mathrm{sepset}(\G_l, \{V_i,V_k\})$}{

orient any $V_i\any V_k$ as $V_i\rightany V_k$, and any $V_k\any V_j$ to $V_k\leftany V_j$, in $\G$
} 

} 

} 
\For{each pair of adjacent nodes $V_i$ and $V_j$ in $\G$}{

\If{for all $l\in\{1,\ldots ,n\}\\ \{ V_i, V_j\}\cup \left(\adj{\G}{V_i}\cap\mathrm{past}^\tau_\V(V_i)\right)\cup
\left((\possdsepp{\G}{V_i,V_j}\cup\possdsepp{\G}{V_j,V_i})\cap\mathrm{past}^\tau_\V(\{V_i, V_j\})\right)\not\subset\V_l$ and\\ $\{ V_i, V_j\}\cup\left(\adj{\G}{V_j}\cap\mathrm{past}^\tau_\V(V_j)\right)\cup
\left((\possdsepp{\G}{V_i,V_j}\cup\possdsepp{\G}{V_j,V_i})\cap\mathrm{past}^\tau_\V(\{V_i, V_j\})\right)\not\subset\V_l$}{

add $\{V_i\any V_j\}$ to \textbf{RemoveEdges}
} %

} 

\end{algorithm}

\setcounter{algocf}{2}

\begin{algorithm}

\setcounter{AlgoLine}{37}

\caption{The tiered IOD (tIOD) algorithm (continued)}
\label{alg:tiod}
\small
\Input{Node sets $\V_1,\ldots, \V_n$, tiered ordering $\tau$, and independence oracles over $\X_{\V_1},\ldots,\X_{\V_n}$} 

\Output{set of partial mixed graphs $\mathbfcal{G}$}

\For{each $\E'\in\mathscr{P}(\mathbf{RemoveEdge})$}{

\For{each $V_j\in\V$ and pair of nodes $V_i,V_k\in\adj{\G_{\E\backslash\E'}}{V_j}$}{

\If{the triple $\langle V_i, V_j, V_k\rangle$ can be oriented as a v-structure in $\G_{\E\backslash\E'}$ and for every $l\in\{1,\ldots ,n\}$ either $V_j\notin\V_l$ or $\mathrm{sepset}(\G_l, \{V_i, V_k\}$ is undefined}{

\uIf{$\tau(V_j)>\max(\tau(V_i),\tau(V_j))$}{

orient $\langle V_i, V_j, V_k\rangle$ as a v-structure

}\Else{

add $\langle V_i, V_j, V_k\rangle$ to $\mathbf{OrientVstructure}$
}
} 

} 

\For{each $\V'\in\mathscr{P}(\mathbf{OrientVstructure})$}{

$\G_{\E\backslash\E'}^{\V'}=\G_{\E\backslash\E'}$ 

orient every triple in $\V'$ as a v-structure in $\G_{\E\backslash\E'}^{\V'}$

\Repeat{all graphs are closed under orientation rules R1-R4 and R8-R10}{

apply orientation rules R1-R4 and R8-R10

\If{a discriminating path is found in R4}{

construct two graphs from $\G_{\E\backslash\E'}^{\V'}$ with each direction of the discriminated collider and continue orienting edges in both graphs
} 

} 

add all graphs to $\mathbf{Possible}\mathbfcal{G}$

\For{each $\G'=(\V,\E'')\in \mathbf{Possible}\mathbfcal{G}$}{

$\M=\G'$

$\mathbfcal{G}'=\{\G'\}$

\For{each $\{V_i\rightpart V_j\}\in\E''$}{

orient $V_i\rightpart V_j$ as $V_i\rightarrow V_j$ in $\M$
} 

orient $\M_U$ as a DAG with no new v-structures and replace every edge $\{V_i\nondir V_j\}$ in $\M$ with the corresponding edge from $\M_U$

\If{(i) $\M$ is a MAG, and (ii) each set in $\mathrm{sepset}$ corresponds to an independence in $\I(\M)$,  (iii) for every $l$ for each $\langle \{V_i, V_j\}, \V_l\rangle\in\mathbf{InducingPaths}$ there is an inducing path between $V_i$ and $V_j$ with respect to $\V\backslash\V_l$ in $\M$, and (iv) if for every non-adjacent pair $A$ and $B$ in $\M$ there exists a set $\mathbf{S}’\subseteq \mathrm{past}^\tau_\V(A,B)$ such that $A$ and $B$ are m-separated by $\mathbf{S}’$ in $\M$}{

\texttt{\#omit lines 67-72 in the simple tIOD algorithm}

\For{each ordered pair of adjacent nodes $V_i, V_j$ in $\G'$}{

\If{$\tau(V_i)<\tau(V_j)$}{

orient $V_i\leftpartany V_j$ as $V_i\rightany V_j$

} 

} 

apply rules R1-R4 and R8-R10673 to $\G'$ until none applies. If a discriminating path is found then construct two graphs with each direction of the discriminated collider and continue orienting edges in both graphs. Let $\mathbfcal{G}'$ be the set of these graphs.

add every $\G'\in\mathbfcal{G}'$ to $\mathbfcal{G}$

}

} 

} 

} 
    
\end{algorithm}

\begin{remark}
    The tIOD algorithm may also be made more efficient by, e.g., not considering nodes in $\adj{\G}{A}$ and $\adj{\G}{B}$ at lines 33-36 of Algorithm \ref{alg:tiod} when determining whether an edge $A\any B$ is removable. However, we focus on the gain in efficiency by adding tiered background knowledge, and we choose to otherwise follow what is done in the original IOD algorithm.
\end{remark}

\section{Number of graphs output in Example \ref{ex:iod} and Example \ref{ex:tiod}}\label{sec:example}

\begin{table}[!htbp]
    \centering
    \begin{tabular}{l | m{4.25em} | m{4.5em} | m{4.25em}}

    $\E'\in\mathscr{P}(\mathbf{RemoveEdges})$ & Possible v-struct. & Graphs considered & Graphs output\\ 
    \hline\hline
$\emptyset$ & 1 & 2 & 1\\ 
\hline
$\{A\nondir D\}$ & 3 & 8 & 4\\ 
$\{A\nondir B\}$ & 2 & 4 & 0\\
$\{A\nondir C\}$ & 2 & 4 & 0\\
$\{B\nondir D\}$ & 1 & 2 & 1\\
$\{C\nondir D\}$ & 1 & 2 & 1\\
\hline
$\{\{A\nondir D\},\{A\nondir B\}\}$ & 2 & 4 & 0 \\
$\{\{A\nondir D\},\{A\nondir C\}\}$ & 2 & 4 & 0 \\
$\{\{A\nondir D\},\{B\nondir D\}\}$ & 1 & 2 & 0\\
$\{\{A\nondir D\},\{C\nondir D\}\}$ & 1 & 2 & 0\\
$\{\{A\nondir B\},\{A\nondir C\}\}$ & 3 & 5 & 0 \\
$\{\{A\nondir B\},\{B\nondir D\}\}$ & 0 & 1 & 0 \\
$\{\{A\nondir B\},\{C\nondir D\}\}$ & 2 & 4 & 0 \\
$\{\{A\nondir C\},\{B\nondir D\}\}$ & 2 & 4 & 0 \\
$\{\{A\nondir C\},\{C\nondir D\}\}$ & 0 & 1 & 0 \\
$\{\{B\nondir D\},\{C\nondir D\}\}$ & 2 & 4 & 1\\
\hline
$\{\{A\nondir D\},\{A\nondir B\}, \{A\nondir C\}\}$ & 1 & 2 & 0\\
$\{\{A\nondir D\},\{A\nondir B\}, \{B\nondir D\}\}$ & 1 & 2 & 0\\
$\{\{A\nondir D\},\{A\nondir B\}, \{C\nondir D\}\}$ & 0 & 1 & 0\\
$\{\{A\nondir D\},\{A\nondir C\}, \{B\nondir D\}\}$ & 0 & 1 & 0 \\
$\{\{A\nondir D\},\{A\nondir C\}, \{C\nondir D\}\}$ & 1 & 2 & 0\\
$\{\{A\nondir D\},\{B\nondir D\}, \{C\nondir D\}\}$ & 0 & 1 & 0\\
$\{\{A\nondir B\},\{A\nondir C\}, \{B\nondir D\}\}$ & 1 & 2 & 0\\
$\{\{A\nondir B\},\{A\nondir C\}, \{C\nondir D\}\}$ & 1 & 2 & 0\\
$\{\{A\nondir B\},\{B\nondir D\}, \{C\nondir D\}\}$ & 1 & 2 & 0\\
$\{\{A\nondir C\},\{B\nondir D\}, \{C\nondir D\}\}$ & 1 & 2 & 0\\
\hline
$\{\{A\nondir B\},\{A\nondir C\}, \{A\nondir D\},\{B\nondir D\}\}$ & 0 & 1 & 0\\
$\{\{A\nondir C\},\{A\nondir D\}, \{B\nondir D\},\{C\nondir D\}\}$ & 0 & 1 & 0\\
$\{\{A\nondir D\},\{B\nondir D\}, \{C\nondir D\},\{A\nondir B\}\}$ & 0 & 1 & 0\\
$\{\{B\nondir D\},\{C\nondir D\}, \{A\nondir B\},\{A\nondir C\}\}$ & 0 & 1 & 0\\
$\{\{C\nondir D\},\{A\nondir B\}, \{A\nondir C\},\{A\nondir B\}\}$ & 0 & 1 & 0\\
\hline
$\E$ & 0 & 1 & 0 \\
\hline \hline
Total & 32 & 73 & 8\\

    \end{tabular}
    \caption{Number of graphs output by the IOD algorithm in Example \ref{ex:iod}}
    \label{tab:ex1}
\end{table}

\begin{table}[!htbp]
    \centering
    \begin{tabular}{l | m{4.25em} | m{4.5em} | m{4.25em}}

    $\E'\in\mathscr{P}(\mathbf{RemoveEdges})$ & Possible v-struct. & Graphs visited & Graphs output\\ 
    \hline\hline
$\emptyset$ & 0 & 1 & 1\\ 
\hline
$\{A\nondir D\}$ & 2 & 4 & 4\\ 
$\{B\nondir D\}$ & 1 & 2 & 1\\
$\{C\nondir D\}$ & 1 & 2 & 1\\
\hline
$\{\{A\nondir D\},\{B\nondir D\}\}$ & 1 & 2 & 0\\
$\{\{A\nondir D\},\{C\nondir D\}\}$ & 1 & 2 & 0 \\
$\{\{B\nondir D\},\{C\nondir D\}\}$ & 2 & 4 & 1 \\
\hline

$\{\{A\nondir D\},\{B\nondir D\}, \{C\nondir D\}\}$ & 0 & 1 & 0\\

\hline \hline
Total & 8 & 18 & 8\\

    \end{tabular}
    \caption{Number of graphs visited and output by the tIOD algorithm in Example \ref{ex:tiod}}
    \label{tab:ex2}
\end{table}

\section{Proofs}\label{app:proofs}

\renewcommand*{\proofname}{Proof of Proposition \ref{prop:past}}

\begin{proof}
The ``if'' part is trivial, we show the ``only if'' part. Assume $\G$ is a DAG. Then $A$ and $B$ can be d-separated in $\G$ by some $\mathbf{S}\subset\V$ if and only if they can be d-separated by $\pa{\G}{A}$ or $\pa{\G}{B}$ \citep{verma1990equivalence, spirtes1991algorithm}. Since $\pa{\G}{A}\subseteq\mathrm{past}^{\tau}_\V(A)$ and $\pa{\G}{B}\subseteq\mathrm{past}^{\tau}_\V(B)$ for any $\tau$ consistent with $\G$, if $A$ and $B$ cannot be d-separated by any $\mathbf{S}\subseteq\mathrm{past}^{\tau}_\V(A)$ or $\mathbf{S}'\subseteq\mathrm{past}^{\tau}_\V(B)$, then they cannot be d-separated by $\pa{\G}{A}$ or $\pa{\G}{B}$, and then no subset of $\V$ d-separates them. Assume instead that $\G$ is a MAG.  Then $A$ and $B$ can be m-separated in $\G$ by any $\mathbf{S}\subset\V$ if and only if they can be m-separated by $\dsep{\G}{A,B}$ or $\dsep{\G}{B,A}$  by Lemma \ref{lemma:dsep}. Recall that $\dsep{\G}{A,B}\subseteq\an{\G}{\{ A,B\}}$ (Definition \ref{def:dsep}). Note that $\an{\G}{A}\subseteq\mathrm{past}^\tau_\V(A)$ and $\an{\G}{B}\subseteq\mathrm{past}^\tau_\V(B)$ for any $\tau$ consistent with $\G$, and it follows that $\dsep{\G}{A,B}\subseteq \mathrm{past}^\tau_\V(A)\cup\mathrm{past}^\tau_\V(B)$. Assume without loss of generality that $\tau(A)\leq\tau(B)$, then $\mathrm{past}^\tau_\V(A)\subseteq\mathrm{past}^\tau_\V(B)$ and then $\dsep{\G}{A,B}\subseteq\mathrm{past}^\tau_\V(B) $: If there is no $\mathbf{S}'\subseteq\mathrm{past}^{\tau}_\V(B)$ that m-separates $A$ and $B$, then they are not m-separated by $\dsep{\G}{A,B}$, and then no subset of $\V$ m-separates them. The same holds for $\dsep{\G}{B,A}$.
\end{proof}

\renewcommand*{\proofname}{Proof}

\begin{lemma}\label{lemma:possdsep}
    Let $\M=(\V,\E)$ be a MAG and $A$ and $B$ distinct nodes in $\V$. Let $\G$ be the graph obtained at line 22 of the tFCI algorithm (Algorithm \ref{alg:tfci}) when estimating the equivalence class of $\M$ using oracle knowledge. Then $A$ and $B$ are adjacent in $\M$ if and only if they are adjacent in $\G$.
\end{lemma}

\begin{proof}
    First, we could have obtained the same graph at line 11 if we had run the algorithm without using background knowledge c.f. Proposition \ref{prop:past}, since we have oracle knowledge of the conditional independences and correct background knowledge. 
    Second, note that given oracle knowledge, any v-structure oriented at line 14 based on the tiered ordering would also have been oriented without background knowledge. Since we obtain the same graph at line 16 as we would have without background knowledge, by Lemma 13 in \citet{spirtes1999fci} $\dsep{\M}{A,B}\subseteq\possdsep{\G}{A,B}$, and by Lemma 1.2 in \citet{colombo2012learning}, $A$ and $B$ can be m-separated by any set in $\M$ if and only if they are m-separated by a set of nodes that all lie on a path between $A$ and $B$ in $\M$, i.e. $\dsep{\M}{A,B}\subseteq\possdsepp{\G}{A,B}$. By Proposition \ref{prop:past}, $A$ and $B$ can be m-separated in $\M$ if and only if they are m-separated by a subset of $\mathrm{past}^\tau_\V(\{A,B\})$ for any $\tau$ consistent with $\M$. Hence, $\dsep{\M}{A,B}\subseteq\possdsepp{\G}{A,B}\cap\mathrm{past}^\tau_\V(\{A,B\})$ and every potential separating set is checked. Hence, $A$ and $B$ are adjacent in $\M$ if and only if they are adjacent in $\G$.
\end{proof}

\renewcommand*{\proofname}{Proof of Proposition \ref{prop:simpletfci}}

\begin{proof}
    Let $\M$ be a MAG and let $\p$ be the PAG representing the equivalence class of $\M$. We show that if we use oracle knowledge of the conditional independencies implied by $\M$ as input, then the simple tFCI algorithm  outputs $\p$. By Lemma \ref{lemma:possdsep}, the skeleton obtained at line 22 is the skeleton of $\M$. Due to the background knowledge being correct, no v-structure, that could not have been found from the conditional independencies, is oriented by background knowledge. Since we have oracle knowledge of the conditional independencies, every v-structure that cannot be oriented from background knowledge will be oriented by the conditional independencies. The orientation rules R1-R4 and R8-R10 are sound and complete given the correct adjacencies and v-structures \citep{spirtes1999fci, zhang2008completeness}, hence the graph $\G$ output by the algorithm is the PAG $\p$ representing the equivalence class of $\M$.
\end{proof}

\renewcommand*{\proofname}{Proof of Proposition \ref{prop:fulltfci}}

\begin{proof}
     Let $\M$ be a MAG. We show that if we use oracle knowledge of the conditional independencies implied by $\M$ and correct background knowledge as input the full tFCI algorithm outputs a PMG $\G$ that  does not contain any arrowhead or tail that is not in $\M$. By Proposition \ref{prop:simpletfci}, the graph obtained at line 28 has the same skeleton and v-structures as $\M$. We only need to argue that the additional arrowheads and tails obtained in the full tFCI algorithm are correct. The arrowhead orientation in line 32 is sound due the background knowledge being correct. The orientation rules R1-R4 and R8-R10 are sound in the sense that they prevent the encoded independence model to change, any new cycles to occur, or any orientations that  contradict the ancestral relations encoded \citep{spirtes1999fci, zhang2008completeness}. Hence, they do not introduce any new information not already known from the conditional independencies and background knowledge, and the graph output after applying these rules does not contain any arrowheads or tails not in $\M$.
\end{proof}

\renewcommand*{\proofname}{Proof}

In order to prove Proposition \ref{prop:simpletiod}, we need some extensions of the results in \citet{tillman2011learning}, and we follow their proof strategy of their Theorem 5.1 and 5.2 to a large extend. Here, Corollary \ref{cor:sep} (which follows from Proposition \ref{prop:simpletfci}) is analogous to Theorem 7.2, Corollary \ref{cor:sep2} extends Corollary 7.1, Lemma \ref{lemma:sep} extends Lemma 7.1, Lemma \ref{lemma:adj} extends Lemma 7.6, and Lemma \ref{lemma:vstruct} extends Lemma 7.7.

\begin{corollary}\label{cor:sep}
    Let $\mathrm{sepset}$ be constructed as in the simple tIOD algorithm (Algorithm \ref{alg:tiod}), and for all $i\in\{1,\ldots ,n\}$ let $\M_i=(\V_i,\E_i)$ be a MAG. For every $i\in\{1,\ldots ,n\}$ if for all $A,B\in\V_i$ such that $\mathrm{sepset}(\G_i, \{A,B\})$ is defined $A\perp_{\M_i} B\mid\mathrm{sepset}(\G_i, \{A,B\})$, then for all $A,B\in\V_i$ and $\mathbf{S}\subseteq\V_i\backslash\{A,B\}$: $X_A\indep X_B\mid \mathbf{X}_\mathbf{S}\Rightarrow A\perp_{\M_i} B\mid \mathbf{S}$
\end{corollary}

From this we have the following result.

\begin{corollary}\label{cor:sep2}
    Let $\mathrm{sepset}$ be constructed as in the simple tIOD algorithm (Algorithm \ref{alg:tiod}), and let $\M=(\V,\E)$ be a MAG. Then for all $i\in\{1,\ldots ,n\}$ if for all $A,B\in\V_i$ such that $\mathrm{sepset}(\G_i, \{A,B\})$ is defined, $A\perp_{\M} B\mid\mathrm{sepset}(\G_i, \{A,B\})$, then for all $A,B\in\V_i$ and $\mathbf{S}\subseteq\V_i\backslash\{A,B\}$: $X_A\indep X_B\mid \mathbf{X}_\mathbf{S}\Rightarrow A\perp_{\M} B\mid \mathbf{S}$
\end{corollary}

\begin{proof}
    This proof is identical to the proof of Corollary 7.1 in \citet{tillman2011learning}.
\end{proof}

\begin{lemma}\label{lemma:sep}
Let $\G'$ be a graph that is added to $\mathbfcal{G}$ at line 73 in the simple tIOD algorithm (Algorithm \ref{alg:tiod}), and let $\M$ be the graph obtained at line 64. Then $\M$ is a MAG and for all $i\in\{1,\ldots ,n\}$, $A,B\in\V_i$ and $\mathbf{S}\subseteq\V_i\backslash\{A,B\}$ $X_A\indep X_B\mid \mathbf{X}_\mathbf{S}\Leftrightarrow A\perp_\M B\mid \mathbf{S}$.
\end{lemma}

\begin{proof}
    This proof directly follows from the proof of Lemma 7.1 in \citet{tillman2011learning}. $\M$ is a MAG by condition (i) at line 65. By condition (ii) at line 65 and Corollary \ref{cor:sep2}, $X_A\indep X_B\mid \mathbf{X}_\mathbf{S}\Rightarrow A\perp_\M B\mid \mathbf{S}$. Then, by condition (iii) at line 65 and Theorem 3.2 in \citet{tillman2011learning} it follows that $X_A\indep X_B\mid \mathbf{X}_\mathbf{S}\Leftrightarrow A\perp_\M B\mid \mathbf{S}$.
\end{proof}

\begin{lemma}\label{lemma:same}
    Let $\G_1,\ldots , \G_n$, $\G$ and $\mathbf{InducingPaths}$ be the graphs and sets of dependent pairs obtained at line 32 of the oracle version of the simple tIOD algorithm (Algorithm \ref{alg:tiod}). Then these are identical to the graphs $\G_1,\ldots , \G_n$, $\G$ and and sets of dependent pairs $\mathbf{InducingPaths}$ obtained after line 30 of Algorithm 2 in \citet{tillman2011learning} (the IOD algorithm).
\end{lemma}

\begin{proof}
We  refer to Algorithm \ref{alg:tiod} as the tIOD (algorithm) and Algorithm 2 in \citet{tillman2011learning} as the IOD (algorithm).

Consider a fixed $i\in\{1,\ldots,n\}$, let $A,B\in \V_i$ and let $\possdsepp{\G_i}{A,B}^\tau$ and $\possdsepp{\G_i}{B,A}^\tau$  be sets of nodes considered at line 20 of the tIOD algorithm. Let $\M_i$ be the true MAG over $V_i$. Then, $A$ and $B$ are m-separated in $\M_i$ if and only if they are m-separated by $\dsep{\M_i}{A,B}$ or $\dsep{\M_i}{B,A}$ (Definition \ref{def:dsep}) by Lemma \ref{lemma:dsep}. Since $\dsep{\M_i}{A,B}\subseteq\possdsepp{\G_i}{A,B}^\tau\cap\mathrm{past}^\tau_\V(A,B)$ and $\dsep{\M_i}{B,A}\subseteq\possdsepp{\G_i}{B,A}^\tau\cap\mathrm{past}^\tau_\V(A,B)$ (c.f. proof of Lemma \ref{lemma:possdsep}), the correct skeleton is obtained a line 26. Let $\possdsepp{\G_i}{A,B}$ and $\possdsepp{\G_i}{B,A}$ be sets of nodes considered at line 18 of the IOD algorithm. Then $\possdsepp{\G_i}{A,B}^\tau\cap\mathrm{past}^\tau_\V(A,B)\subseteq \possdsepp{\G_i}{A,B}$ and $\possdsepp{\G_i}{B,A}^\tau\cap\mathrm{past}^\tau_\V(A,B)\subseteq \possdsepp{\G_i}{B,A}$. Hence, the tIOD and IOD  obtain the same skeleton of $\G_i$, and the same $\mathbf{InducingPaths}$ for $V_i$. Since this must be the case for every $i$, the two algorithms  also obtain the same $\mathbf{InducingPaths}$ and the same skeleton of $\G$ after line 26 of the tIOD algorithm and line 24 of the IOD algorithm.

Next, the algorithms differ at lines 27-32 of the tIOD  (line 25-29 in the IOD). In the tIOD, some arrowheads are oriented based on the tiered ordering, and some are oriented based on the conditional independencies. However, since the tiered background knowledge is assumed to be correct, any orientation by background knowledge could have been oriented solely by oracle knowledge of the conditional independencies. Here, the IOD orients arrowheads only based on conditional independencies. However, under oracle knowledge, both algorithms use the same conditional independencies, and since every $\G_i$ has the same skeleton in both algorithms, they orient the same arrowheads in $\G$.
\end{proof}

\begin{lemma}\label{lemma:adj}
    Let $\M$ be a MAG over $\V=\V_1\cup\ldots\cup\V_n$ such that for all $i\in\{1,\ldots ,n\}$ $A,B\in\V_i$ and $\mathbf{S}\subseteq\V_i\backslash\{A,B\}$ $A\perp_\M B\mid \mathbf{S}\Leftrightarrow X_A\indep X_B\mid\X_\mathbf{S}$, then some graph considered at line 38 of Algorithm \ref{alg:tiod} has the same skeleton as $\M$.
\end{lemma}

\begin{proof}
    Let $\G$  be the graph considered at line 33 of Algorithm \ref{alg:tiod}. By Lemma \ref{lemma:same} it follows from Lemma 7.4 in \citet{tillman2011learning} that $\G$  contains a superset of the adjacencies in $\M$. Let $A$ and $B$ be  a pair of nodes that are adjacent in $\G$ but not in $\M$. Then, by maximality of $\M$, there is some set $\mathbf{S}\subseteq\V\backslash\{A,B\}$ that m-separates $A$ and $B$ in $\M$. It then follows from Corollary \ref{cor:sep2} that for all $i$, no such  $\mathbf{S}$ is a subset of $\V_i$. One such $\mathbf{S}$ is in $\adj{\M}{A}\cap\mathrm{past}^\tau_\V(A)\subseteq\adj{\G}{A}\cap\mathrm{past}^\tau_\V(A)$ or $\left( \possdsepp{\M}{A,B}\cup\possdsepp{\M}{B,A}\right)\cap\mathrm{past}^\tau_\V(\{A, B\})\subseteq\left( \possdsepp{\G}{A,B}\cup\possdsepp{\G}{B,A}\right)\cap\mathrm{past}^\tau_\V(\{A, B\})$, or $\mathbf{S}$ is in $\adj{\M}{B}\cap\mathrm{past}^\tau_\V(B)\subseteq\adj{\G}{B}\cap\mathrm{past}^\tau_\V(B)$ or \\$\left( \possdsepp{\M}{A,B}\cup\possdsepp{\M}{B,A}\right)\cap\mathrm{past}^\tau_\V(\{A, B\})\subseteq\left( \possdsepp{\G}{A,B}\cup\possdsepp{\G}{B,A}\right)\cap\mathrm{past}^\tau_\V(\{A, B\})$. Hence, $A$ and $B$ are added to $\mathbf{RemoveEdge}$. Since all combinations of edge removals are considered, one graph must have the same skeleton as $\M$.
\end{proof}

\begin{lemma}\label{lemma:vstruct}
    Let $\M$ be a MAG over $\V$ such that for all $i\in\{1,\ldots ,n\}$ $A,B\in\V_i$ and $\mathbf{S}\subseteq\V_i\backslash\{A,B\}$ $A\perp_\M B\mid \mathbf{S}\Leftrightarrow X_A\indep X_B\mid\X_\mathbf{S}$, then some graph considered at line 49 of Algorithm \ref{alg:tiod} has the same v-structures as $\M$.
\end{lemma}

\begin{proof}
Let $\G$ be a graph constructed at line 38 of Algorithm \ref{alg:tiod} and assume that $\G$ has the same skeleton as $\M$ (such a $\G$ exists by Lemma \ref{lemma:adj}). By Lemma \ref{lemma:same} it follows from Lemma 7.5 in \citet{tillman2011learning} that $\G$ contains a subset of the v-structures in $\M$. Let $\langle A, B, C\rangle$ be an unshielded triple in $\G$ and $\M$, and assume that this is a v-structure in  $\M$ but not in $\G$. If $\tau(B)>\max(\tau(A), \tau(C))$, then every graph considered at line 49 must have  $\langle A, B, C\rangle$ oriented as a v-structure. If $\tau(B)\leq\max(\tau(A), \tau(C))$ and there were an $i$ where $B\in\V_i$ and $\mathrm{sepset}(\G_i, \{A,C\})$ were defined, then the v-structure would have been oriented in $\G$ (line 29 of Algorithm \ref{alg:tiod}).   Hence, $\langle A, B, C\rangle$ is added to $\mathbf{OrientVstructure}$. Since all combinations of possible v-structures are considered, one graph must have the same v-structures as $\M$. 
\end{proof}

\begin{lemma}\label{lemma: consistent}
    Let $\G$ be a PAG. Then $\G$ is consistent with the tiered ordering $\tau$ if and only if some $\M\in [\G]$ satisfies criterion (iv) at line 65 in Algorithm \ref{alg:tiod}.
\end{lemma}

\begin{proof}
$\G$ is consistent with $\tau$ if there exist an $\M\in[\G]$ that encodes the background knowledge implied by $\tau$. Let $\M\in[\G]$. 

     ``\emph{If}'': Let $A$ and $B$ non-adjacent in $\M$. If for every path $\pi$ between $A$ and $B$ in $\M$ that only goes through nodes in  $\V\backslash\mathrm{past}_\V^\tau(\{A,B\})$, $\pi$ is not m-connecting given  $\mathrm{past}_\V^\tau(\{A,B\})$, then the orientation of cross-tier edges in $\G$, which is the PAG of $\M$, will not construct any new v-structures.

      ``\emph{Only if}'': Let $A$ and $B$ be nodes in $\M$ and assume that (iv) at line 65 is not satisfied. Let $\pi$ be a path between $A$ and $B$ that goes through nodes in $\V\backslash\mathrm{past}_\V^\tau(\{A,B\})$. Assume that $\pi$ is m-connecting given every $\mathbf{S}\subset \mathrm{past}^\tau_\V(\{A,B\})$. Then, $\pi$ will have either a directed edge that is into $A$, one that is into $B$, or both. However, $\tau$ implies that the edge between $A$ and the next node on $\pi$ cannot be a directed edge into $A$ (similar for $B$). I.e. $\M$ does then not encode the tiered background knowledge implied by $\tau$. Since such a path exists for any MAG that is Markov equivalent to $\M$, $\G$ is not consistent with $\tau$.
\end{proof}

\renewcommand*{\proofname}{Proof of Proposition \ref{prop:simpletiod}}

\begin{proof}
Let $\G\in\mathbfcal{G}$. First, it follows directly from Lemma 7.3 in \citet{tillman2011learning} that $\G$ is a PAG. Second, $\G$ is only added to $\mathbfcal{G}$ if it satisfies criterion (iv) at line 65. Hence, by Lemma \ref{lemma: consistent} it follows that $\G$ is consistent with $\tau$. We need to show the following:

\begin{itemize}[align=left]
    \item[Soundness:] For all PAGs $\G\in\mathbfcal{G}$  and for all $i\in\{1,\ldots,n\}$ $A,B\in V_i$ and $\mathbf{S}\subseteq V_i\backslash\{ A,B\}$ for all $\M\in[\G]$ $A\perp_\M B\mid \mathbf{S}\Leftrightarrow X_A\indep X_B\mid\X_\mathbf{S}$.
    
    \item[Completeness:] Let $\M$ be a MAG over $\V$ such that the PAG of $\M$ is consistent with $\tau$ and for all $i\in\{1,\ldots ,n\}$ $A,B\in\V_i$ and $\mathbf{S}\subseteq\V_i\backslash\{A,B\}$ $A\perp_\M B\mid \mathbf{S}\Leftrightarrow X_A\indep X_B\mid\X_\mathbf{S}$, then $\M\in[\G]$ for some $\G\in\mathbfcal{G}$.

\end{itemize}

Soundness:  Let $\M$ be the graph obtained at line 62. By Theorem 2 in \citet{zhang2008completeness} $\M$ is a MAG, and by Lemma \ref{lemma:sep} it holds that for all $i\in\{1,\ldots,n\}$, $A,B\in\V_i$ and $\mathbf{S}\subseteq\V_i\backslash\{A,B\}$: $X_A\indep X_B\mid\mathbf{X}_\mathbf{S}\Leftrightarrow A\perp_\M B\mid \mathbf{S}$. Since $\M\in[\G]$, for any other $\M'\in[\G]$ it  also holds that $X_A\indep X_B\mid\mathbf{X}_\mathbf{S}\Leftrightarrow A\perp_{\M'} B\mid \mathbf{S}$.

    Completeness: 
    We follow the proof of Theorem 5.2 in \citet{tillman2011learning}.  By Lemma \ref{lemma:vstruct} there is a graph $\G'$ considered at line 50 of Algorithm \ref{alg:tiod} that has the same skeleton and v-structures as $\M$. Assume that $\G'$ also contains arrowheads or tails that are not in $\M$: Then this has been oriented at line 27. Following the proof of Theorem 5.2 in \citet{tillman2011learning}, this must be an arrowhead that is into the collider in a v-structure in some MAG $\M_i=(\V_i,\E_i)$, for some $i\in\{1,\ldots,n\}$, such that $\M$ and $\M_i$ encode the same m-separations over $\V_i$, and by Corollary 7.2 in \citet{tillman2011learning}, this arrowhead must then be contained in the PAG representing the equivalence class of $\M$. Hence, this orientation must also be in $\M$. By the soundness and completeness of the orientation rules R1-R4 and R8-R10 \citep{zhang2008completeness}, the PAG $\G$ of $\M$ is contained in $\mathbf{Possible}\mathbfcal{G}$. By Theorem 2 in \citet{zhang2008completeness}, given $\G$, at line 62 we obtain a MAG that is Markov equivalent to $\M$. For this MAG, criteria (i) and (ii) at line 65 are satisfied, and by Theorem 4.2 in \citet{richardson2002ancestral} criterion (iii) is also satisfied. Hence, the PAG of $\M$ is added to $\mathbfcal{G}$.
\end{proof}

\renewcommand*{\proofname}{Proof of Proposition \ref{prop:fulltiod}}

\begin{proof}
Let $\M$ be a MAG over $\V=\V_1\cup\ldots\cup\V_n$ that encodes $\tau$. Consider the full tIOD algorithm using oracle knowledge of the marginal independence models over $\V_1,\ldots,\V_n$ and the background knowledge implied by $\tau$. Let $\G'$ be the graph considered at line 59 and assume that conditions (i), (ii), (iii) and (iv) at line 65 are satisfied. Assume that $\G'$ is the PAG of $\M$ (by Proposition \ref{prop:simpletiod} some $\G'$ will be the PAG of $\M$). We will argue that the additional arrowheads and tails obtained in the full tIOD algorithm are correct: The arrowhead orientation at line 69 is sound since we assumed the tiered background knowledge to be correct. The orientation rules R1-R4 and R8-R10 are sound in the sense that they prevent the encoded independence model to change, any new cycles to occur, or any orientations that would contradict the ancestral relations encoded \citep{spirtes1999fci, zhang2008completeness}. Hence, they do not introduce any new information not already known from the independence models and background knowledge. Moreover, for every possible orientation implied by R4 we include both graphs. Hence, at least one graph output encodes the independence model of $\M$ and does not contain any arrowheads and tails not in $\M$.
\end{proof}

\renewcommand*{\proofname}{Proof of Proposition \ref{prop:informativeness}}

\begin{proof}
We refer to Algorithm \ref{alg:tiod} as the tIOD (algorithm) and Algorithm 2 and 3 in \citet{tillman2011learning} as the IOD (algorithm).

In the first part of the algorithms (line 32 of the tIOD  and line 30 of Algorithm 2 in \citet{tillman2011learning}), the IOD algorithm and simple tIOD obtain the same graph $\G$ over $\V$ (Lemma \ref{lemma:same}). The only other part later in the simple tIOD algorithm where tiered background knowledge is being used is when constructing $\mathbf{RemoveEdge}$ and $\mathbf{OrientVstructure}$. Like the IOD, the tIOD visits all graphs that can be constructed from $\G$ from all combinations of edge removals in $\mathbf{RemoveEdge}$ and v-structured orientations in $\mathbf{OrientVstructure}$. Let $\mathbf{RemoveEdge}$ and $\mathbf{OrientVstructure}$ be obtained in the IOD algorithm (Algorithm 3 in \citet{tillman2011learning}), and let $\mathbf{RemoveEdge}^\tau$ and $\mathbf{OrientVstructure}^\tau$ be obtained in the tIOD algorithm. Then the tIOD algorithm visits fewer graphs than the IOD algorithm if and only if either\\ $|\mathbf{RemoveEdge}|>|\mathbf{RemoveEdge}^\tau|$, $|\mathbf{OrientVstructure}|>|\mathbf{OrientVstructure}^\tau|$, or both $|\mathbf{RemoveEdge}|>|\mathbf{RemoveEdge}^\tau|$  and $|\mathbf{OrientVstructure}|>|\mathbf{OrientVstructure}^\tau|$. We show that this is equivalent to conditions (i) and (ii).

First, note that it must always be the case that $|\mathbf{RemoveEdge}^\tau|\leq|\mathbf{RemoveEdge}|$ since for any $A,B$ 
\begin{align*}
\{A,B\}\cup\left(\adj{\G}{A}\cap\mathrm{past}^\tau_\V(A)\right)&\cup\left(\left(\possdsepp{\G}{A,B}\cup\possdsepp{\G}{B,A}\right)\cap\mathrm{past}^\tau_\V(A,B)\right)\not\subset\V_i\Rightarrow \\
\{A,B\}\ \cup\ &\adj{\G}{A}\cup\possdsepp{\G}{A,B}\cup\possdsepp{\G}{B,A}\not\subset\V_i
\end{align*}
and
\begin{align*}
\{A,B\}\cup\left(\adj{\G}{B}\cap\mathrm{past}^\tau_\V(B)\right)&\cup\left(\left(\possdsepp{\G}{A,B}\cup\possdsepp{\G}{B,A}\right)\cap\mathrm{past}^\tau_\V(A,B)\right)\not\subset\V_i\Rightarrow \\
\{A,B\}\ \cup\ &\adj{\G}{B}\cup\possdsepp{\G}{A,B}\cup\possdsepp{\G}{B,A}\not\subset\V_i
\end{align*}
Moreover, it also always holds that $|\mathbf{OrientVstructure}^\tau|\leq|\mathbf{OrientVstructure}|$: If a triple is added to $\mathbf{OrientVstructure}^\tau$ it is also added to $\mathbf{OrientVstructure}$ since any unshielded triple obtained from $\mathbf{RemoveEdge}^\tau$ will also be obtained from $\mathbf{RemoveEdge}$.

\emph{``If:''} \emph{Condition (i):}: Since A and B have been measured together in at least one dataset, $X_A$ and $X_B$ are not marginally independent, so if the edge between $A$ and $B$ can be removed, it must be because there exists an $\mathbf{S}\subseteq\V$ that separates $A$ and $B$. The IOD algorithm searches for such a set in $\adj{\G}{A}\cup\possdsepp{\G}{A,B}\cup\possdsepp{\G}{B,A}$ and $\adj{\G}{B}\cup\possdsepp{\G}{A,B}\cup\possdsepp{\G}{B,A}$. If for all $i$ 
 \begin{align*}
     \{A\}\cup\adj{\G}{A}\cup\possdsepp{\G}{A,B}\cup\possdsepp{\G}{B,A}\not\subset\V_i\text{ and }\\ \{B\}\cup\adj{\G}{B}\cup\possdsepp{\G}{A,B}\cup\possdsepp{\G}{B,A}\not\subset\V_i
 \end{align*}
then $A\any B$ is added to $\mathbf{RemoveEdge}$. If for some $i$ $A,B\in\V_i$ and
\begin{align*}
    \left(\possdsepp{\G}{A,B}\cup\possdsepp{\G}{B,A}\right)\backslash\V_i\subseteq\V\backslash\mathrm{past}^{\tau}_\V(A,B)\text{ and }
    \adj{\G}{A}\backslash\V_i\subseteq\V\backslash\mathrm{past}^\tau_\V(A)
\end{align*}
then 
\begin{align*}
    \left(\possdsepp{\G}{A,B}\cup\possdsepp{\G}{B,A}\right)\cap\mathrm{past}^{\tau}_\V(A,B)\subseteq\V_i\text{ and } \adj{\G}{A}\cap\mathrm{past}^\tau_\V(A)\subseteq\V_i
\end{align*} 
 and it then follows that 
\begin{align*}
    \{A,B\}\cup\left(\left(\possdsepp{\G}{A,B}\cup\possdsepp{\G}{B,A}\right)\cap\mathrm{past}^{\tau}_\V(A,B)\right)\cup\left(\adj{\G}{A}\cap\mathrm{past}^\tau_\V(A)\right)\subseteq\V_i
\end{align*}
and $A\any B$ is not added to $\mathbf{RemoveEdge}^\tau$ and $|\mathbf{RemoveEdge}|>|\mathbf{RemoveEdge}^\tau|$. The same holds for $\adj{\G}{B}\backslash\V_i\subseteq\V\backslash\mathrm{past}^\tau_\V(B)$.

\emph{Condition (ii):}  Now, let $\langle A, C, B\rangle$ be an unshielded triple where for all $i$, either $C\notin\V_i$ or $\mathrm{sepset}(\G_i,\{A, B\})$ is undefined. Then $\langle A, C, B\rangle$ is added to $\mathbf{OrientVstructure}$. However, if $\tau(C)>\max(\tau(A),\tau(B))$, this is oriented as a v-structure in the tIOD algorithm (line 42 in Algorithm \ref{alg:tiod}), and not added to $\mathbf{OrientVstructure}^\tau$ and $\mathbf{OrientVstructure}|>|\mathbf{OrientVstructure}^\tau|$.

\emph{``Only if:''} Assume that neither is satisfied. \emph{Condition (i):} Assume that for all adjacent $A$ and $B$ for which it holds that for all $i$
 \begin{align}\label{proof11}
     \{A\}\cup\adj{\G}{A}\cup\possdsepp{\G}{A,B}\cup\possdsepp{\G}{B,A}\not\subset\V_i\text{ and } 
\end{align}
\begin{align*}
     \{B\}\cup\adj{\G}{B}\cup\possdsepp{\G}{A,B}\cup\possdsepp{\G}{B,A}\not\subset\V_i
\end{align*}
it is either the case that $i$ $A$ and $B$ are not in $\V_i$, or
$$\left(\possdsepp{\G}{A,B}\cup\possdsepp{\G}{B,A}\right)\backslash \V_i\not\subset\V\backslash\mathrm{past}^\tau_\V(A,B)\quad\text{or}$$
$$\adj{\G}{A}\backslash\V_i\not\subset\V\backslash\mathrm{past}_\V^\tau(A)\,\,\text{ and }\,\, \adj{\G}{B}\backslash\V_i\not\subset\V\backslash\mathrm{past}_\V^\tau(B)$$
then it follows that
$$\{A,B\}\cup\left(\adj{\G}{A}\cap\mathrm{past}^\tau_\V(A)\right)\cup\left((\possdsepp{\G}{A,B}\cup\possdsepp{\G}{B,A})\cap\mathrm{past}^\tau_\V(\{A, B\})\right)\not\subset\V_i$$ 
and
$$\{A,B\}\cup\left(\adj{\G}{B}\cap\mathrm{past}^\tau_\V(B)\right)\cup\left((\possdsepp{\G}{A,A}\cup\possdsepp{\G}{B,A})\cap\mathrm{past}^\tau_\V(\{A, B\})\right)\not\subset\V_i$$
Hence, for all pairs $A,B$ where (\ref{proof11}) is satisfied $A\any B$ is then added to both $\mathbf{RemoveEdge}$ and $\mathbf{RemoveEdge}^\tau$, and $|\mathbf{RemoveEdge}|=|\mathbf{RemoveEdge}^\tau|$

\emph{Condition (ii):} Assume that for all unshielded triples $\langle A, C, B\rangle$ where for all $i$, either $C\notin\V_i$ or $\mathrm{sepset}(\G_i, \{A, B\})$ is undefined we have that  $\tau(C)\leq\max(\tau(A),\tau(B))$, then all of these triples are added to both $\mathbf{OrientVstructure}$ and $\mathbf{OrientVstructure}^\tau$, and $\mathbf{OrientVstructure}|=|\mathbf{OrientVstructure}^\tau|$.

\end{proof}

\end{document}